\documentclass[%
]{mpi2015-cscpreprint}
\usepackage{geometry}
\usepackage{amsmath}
\usepackage{amsthm}
\usepackage{amssymb}
\usepackage{graphics} 
\usepackage{epsfig} 
\usepackage{amsmath} 
\usepackage{amssymb}  
\usepackage{algorithm,algorithmic}
\usepackage{cite}%
\usepackage{cancel}
\usepackage{mathtools}
\usepackage{subcaption}
\usepackage{graphicx}
\usepackage{enumerate}

\usepackage{tikz,pgfplots}

\pgfplotsset{compat=newest} 
\pgfplotsset{plot coordinates/math parser=false} 
\usepackage{color}

\usepackage[hardware]{mymacros}
\def\totimes{\mathbin{\tilde\otimes}}
\newtheorem{remark}{Remark}
\newtheorem{theorem}{Theorem}

\newtheorem{lemma}{Lemma}

\newtheorem{corollary}{Corollary}

\usepackage{todonotes}

\usepackage[normalem]{ulem}

\usepackage{float}
\usepackage{color}
\usepackage{array}
\usepackage{dsfont}
\usepackage{tikz}
\usetikzlibrary{arrows,decorations.pathmorphing,backgrounds,fit,positioning,shapes.symbols,chains}
\usepackage{cleveref}
\usepackage{wrapfig}
\renewcommand{\hat}[1]{\widehat{#1}}

\newcommand{\pathfig}{./Figures}

\newcommand{\opinfbenchmark}{\texttt{opInf-benchmark}}
\newcommand{\opinf}{\texttt{OpInf}}
\newcommand{\lasmi}{\texttt{lasMI}}
\newcommand{\gasmi}{\texttt{gasMI}}
\newcommand{\atrmi}{\texttt{atrMI}}

\newcommand{\las}{\texttt{LAS}}
\newcommand{\gas}{\texttt{GAS}}
\newcommand{\atr}{\texttt{ATR}}

\newcommand{\dt}{\texttt{dt}}

\begin{document}
\title{Guaranteed Stable Quadratic Models and their Applications in SINDy and Operator Inference}
\author{Pawan Goyal, Igor Pontes Duff, and Peter Benner}

\abstract{
	Scientific machine learning for inferring dynamical systems combines data-driven modeling, physics-based modeling, and empirical knowledge. It plays an essential role in engineering design and digital twinning. In this work, we primarily focus on an \emph{operator inference} methodology that builds dynamical models, preferably in low-dimension, with a prior hypothesis on the model structure, often determined by known physics or given by experts. Then, for inference, we aim to learn the operators of a model by setting up an appropriate optimization problem. 

	One of the critical properties of dynamical systems is \emph{stability}. However, this property is not guaranteed by the inferred models. In this work, we propose inference formulations to learn quadratic models, which are stable by design. Precisely, we discuss the parameterization of quadratic systems that are locally and globally stable. Moreover, for quadratic systems with no stable point yet bounded (e.g., \ chaotic Lorenz model), we discuss how to parameterize such bounded behaviors in the learning process. Using those parameterizations, we set up inference problems, which are then solved using a gradient-based optimization method. Furthermore, to avoid numerical derivatives and still learn continuous systems, we make use of an integral form of differential equations. We present several numerical examples, illustrating the preservation of stability and discussing its comparison with the existing state-of-the-art approach to infer operators. By means of numerical examples, we also demonstrate how the proposed methods are employed to discover governing equations and energy-preserving models. 
}

\novelty{\begin{itemize}
		\item Inference of stable quadratic systems.
		\item Utilize locally and globally stable quadratic system parameterizations.
		\item Discuss an attracting trapping region for quadratic systems and their parametrization.
		\item Make use of an integral form for  the inference of continuous systems so that we do not require estimating the derivative information numerically.
		\item Several numerical examples demonstrate the stability guarantee of the inferred models. 
\end{itemize}}
\maketitle

\section{Introduction}

In recent years, the field of learning dynamical systems from data has gained significant popularity due to the availability of vast amounts of data and its diverse range of applications, including robotics, self-driving cars, epidemiology, neuroscience, and climate modeling. Constructing models for such applications using first principles is often impractical. Leaning dynamical models using data allows to make predictions, optimize system parameters, and enables system control using feedback loops.  
One of the most common approaches to learning dynamical systems from data is based on the Koopman operator theory, which aims at representing the underlying dynamics of a system as a linear operator acting on a high-dimensional space of observables \cite{Mez13}, which in general is even of infinite dimensional. This approach is particularly useful for studying systems with complex and nonlinear dynamics using the tools from linear systems theory. The identification procedure is often done using the \emph{Dynamic Mode Decomposition} (DMD), which has been initially developed in \cite{morSch10}, and several extensions have been proposed, such as extended DMD \cite{williams2015data}, DMD with control input operator \cite{ProBK16, ProBK18}, kernel DMD \cite{WilRK15}, higher-order DMD \cite{LecV17}, and DMD for quadratic-bilinear control systems~\cite{morGosP21}. Furthermore, In the context of nonlinear systems, the so-called Sparse Identification of Nonlinear Dynamics (SINDy)  has gained prominence \cite{brunton2016discovering}. Its underlying principle involves selecting a few nonlinear terms from a vast library of potential candidate nonlinear basis functions to describe the underlying system's dynamics. The basis functions are selected  via sparse regression approaches using the available time-domain data. 
Additionally, in the context of \emph{system identification}, various methodologies have been proposed to identify systems using input-output data in the time domain, including the eigensystem realization algorithm \cite{juang1985eigensystem} and subspace methods \cite{van1994n4sid, viberg1995subspace}. 

The integration of prior knowledge into learning frameworks has become increasingly important for creating accurate and interpretable surrogate models in various applications. This is particularly relevant when we have prior knowledge, such as model hypotheses or a form of governing equations but lack information about system parameters and discretization schemes and only have access to time-domain measurements or simulated data. 
In such scenarios, the operator inference technique has become popular as a valuable tool for learning operators. For high-dimensional data, one can focus on learning dynamics in a low-dimensional subspace, thus yielding reduced-order models. This  problem was initially studied in \cite{morPehW16}, aiming at identifying nonlinear polynomial systems, which has been extended to general nonlinear systems \cite{QKPW2020_lift_and_learn,morBenGKPW20}. Moreover, the methodology has been tailored to capture second-order or mechanical behavior \cite{sharma2022preserving, filanova2022operator}, adapted to incompressible flow problems \cite{morBenGHetal20}, and parametric systems \cite{morYilGBetal20,mcquarrie2021non}.

Physical phenomena often exhibit remarkable stability, with their state variables remaining well-behaved and bounded over long periods. Thus, accurately modeling these processes requires a set of stable differential equations, which are crucial for numerical computations. Despite its importance, stability is often not discussed in frameworks for learning dynamical systems. This paper focuses on learning quadratic differential equations, with guaranteed stability properties. Quadratic models naturally appear in discretized fluid mechanical models. Moreover, smooth nonlinear dynamical systems can be recast as quadratic dynamical systems by means of a lifting transformation \cite{savageau1987recasting,morGu09,morBenB15,QKPW2020_lift_and_learn}. 

In this work, we focus on imposing three distinct stability classes on the learned quadratic models. The first focuses on local (asymptotic) stability around an equilibrium point. The necessary condition for systems to be locally asymptotically stable requires its Jacobian at the equilibrium to be a stable matrix.  
With this spirit, the authors in \cite{sawant2023physics} discussed an inference problem and set up an optimization problem with matrix inequality constraints, which can be computationally expensive. Furthermore, it is assumed that the Jacobian at the equilibrium point  is symmetric and negative definite. But, for many stable systems, the Jacobian is not symmetric. 
To overcome these limitations, in this work, we adopt the stable matrix parametrization from \cite{gillis2017computing} to the inference problem, allowing us to learn quadratic models that are guaranteed to be locally stable by design. Notably, such a parametrization has been previously used in \cite{morGoyPB23} to learn linear dynamical systems.

Next, we examine global (asymptotic) stability as the second type of stability. To investigate global stability, we assume the quadratic nonlinearity is (generalized) energy-preserving \cite{schlegel2015long}. It is worth noting that many dynamical systems in fluid dynamics have energy-preserving nonlinearities, as discussed in \cite{schlegel2015long}. These nonlinearities are commonly seen in discretized fluid mechanical models with a wide range of boundary conditions, such as those presented in \cite{holmes2012turbulence, mccomb1990physics, rummler1998direct, schlichting2016boundary}, as well as in magneto-hydrodynamics applications, as reported in \cite{freidberg2014ideal,galtier2016introduction,kaptanoglu2020two, kaptanoglu2021structure}.
By means of a Lyapunov function, we derive the conditions for quadratic systems with energy-preserving nonlinearities to be globally stable. Consequently, we parametrize a family of energy-preserving quadratic systems that are globally stable. Leveraging this, we then introduce an inference problem to guarantee the learned models to be globally stable by construction.

Besides local asymptotic stability or global stability of a fixed equilibrium point, many nonlinear dynamical systems exhibit attractor behavior, e.g., the chaotic Lorenz example. In this case, the trajectories are bounded and converge (asymptotically) to a bounded region despite not having a stable equilibrium point. Therefore, this work considers a third type of stability, involving the existence of (asymptotic) attracting trapping regions. Energy-preserving quadratic systems possessing attracting trapping regions have been initially studied in \cite{schlegel2015long}. Building upon this work, the authors in \cite{kaptanoglu2021promoting} proposed an extension of the SINDy algorithm that promotes boundedness for quadratic systems. Therein, the proposed algorithm includes energy-preserving nonlinearities as soft constraints and projects the solution at each iteration onto the space of quadratic bounded systems.
In contrast, in this work, we propose a parametrization of quadratic systems possessing trapping regions. This parametrization is then leveraged in the inference problem, ensuring by design that the learned models are always energy-preserving, bounded, and possess trapping regions at each iteration without the need for projection. In other words, our approach encodes the structures that guarantee stability preservation or attractive behavior in strong form, and no projections enforcing the structure are therefore needed.

The remaining paper is structured as follows. In  \Cref{sec:overview}, we briefly recall the operator inference approach \cite{morPehW16} for learning nonlinear systems with quadratic nonlinearities from data. Then, in \Cref{sec:local_Stabil}, we describe local asymptotic stability and introduce a parametrization of locally asymptotically stable quadratic systems. Leveraging this, we propose an inference problem so that the learned models are guaranteed to be locally asymptotically stable.
 Thereafter, global asymptotic stability is studied in \Cref{sec:GlobStab}  for quadratic systems possessing energy-preserving nonlinearity. Then, a novel parametrization of globally asymptotically stable systems is proposed, guaranteeing that the learned models are globally stable. 
In \Cref{sec:atr}, we study energy-preserving quadratic systems possessing bounded trajectories with attracting trapping regions and their parametrization.   It is worth mentioning that all parameterizations proposed assumed that the underlying quadratic dynamical system possesses a quadratic-like Lyapunov function. Then, in \Cref{sec:intergrationscheme}, an integration
scheme is reviewed so that incorporating it in learning models can avoid the requirement of computing the derivative information from data. Finally, \Cref{sec:Exp}  presents several numerical experiments illustrating the guaranteed stabilities of the learned models. By means of two examples, we also illustrate how the proposed parameterization can be used to discover governing equations in the context of SINDy that are stable and can be energy-preserving. In \Cref{sec:Conc}, we conclude the paper with a short summary and future avenues.

\section{Model Inference with Quadratic Non-linearity}\label{sec:overview}
In the following, we provide a brief overview of the standard operator inference (\texttt{OpInf}) approach \cite{morPehW16} by restricting ourselves to quadratic nonlinearities. This approach also closely resembles the \texttt{SINDy} approach \cite{brunton2016discovering} when the dictionary contains only linear and quadratic terms. 
We begin by discussing learning models using high-dimensional data. 

In the following, we formulate the inference problem to obtain low-dimensional dynamical models using high-dimensional data. To that end, we consider a high-fidelity dynamical system with $N$ degrees of freedom, and its state vector is denoted by $\by(t)\in \R^N$. Furthermore, we assume that at time steps $\{t_0, t_1,\ldots,t_\cN\}$, we have access to the snapshots $\by(t_k)$, for $k=1, \dots, \cN$. Next, let us collect these snapshots in a matrix as follows:
\begin{equation}
\bY = \begin{bmatrix} \by(t_0),\ldots, \by(t_\cN) \end{bmatrix} \in \R^{N\times \cN}.
\end{equation}
Though the state $\by(t)$ is $N$-dimensional, so is the dynamics,  it is often possible to approximate it using a lower-dimensional subspace accurately. By doing so, we can significantly simplify the inference problem. To achieve this, we first identify a low-dimensional representation of $\by(t)$ by computing a set of basis vectors  $\bV \in \R^{N\times n}$ of the dominant subspace, which can be obtained using the SVD of the matrix $\bY$, followed by taking its $n$ most dominant left singular vectors. This allows us to  compute the  reduced state trajectory as follows:
\begin{equation}\label{eq:proj_step}
\bX = \bV^\top \bY,
\end{equation}
where $\bX := \begin{bmatrix} \bx(t_0), \ldots, \bx(t_\cN)\end{bmatrix}$ with $\bx(t_i)= \bV^\top \by(t_i).$
With our quadratic model hypothesis, we then aim to learn the operators that have produced the data. Precisely, our goal is to learn a quadratic model of the following form:
\begin{equation}\label{eq:quad_model}
	\dot{\bx}(t) = \bA\bx(t) +  \bH\left(\bx(t)\otimes \bx(t)\right) + \bB, \quad \bx(0) = \bx_0,
\end{equation}
where $\bA \in \R^{n \times n}$, $\bH \in \R^{n\times n^2}$, and $\bB \in \R^{n\times 1}$ are the operators. The quadratic system \eqref{eq:quad_model} is a representative of an important class of nonlinear systems, as a large class of smooth nonlinear systems can be written in this form, see \cite{savageau1987recasting,morGu09,morGu11,morBenB15,morBenG17,morBenGG18,morKraW19,QKPW2020_lift_and_learn}. 

Next, we cast the inference problem, which is as follows. Having the low-dimensional trajectories $\{\bx(t_0), \ldots, \bx(t_\cN)\}$, we aim to learn operators $\bA$, $\bH$, and $\bB$ in \eqref{eq:quad_model}. At the moment, let us also assume to have the derivative information of $\bx$ at time $\{t_0,\ldots,t_\cN\}$, which is denoted by $\dot{\bx}(t_0),\ldots, \dot{\bx}(t_\cN)$. Using this derivative information, we form the following matrix:
\begin{equation}
\dot{\bX} = \begin{bmatrix} 		\dot{\bx}(t_0), \ldots,\dot{\bx}(t_\cN)	\end{bmatrix}.
\end{equation}
Then, in a naive formulation, determining the operators boils down to solving a least-squares optimization problem, which can be written as
\begin{equation}\label{eq:opinf_optimization}
\min_{\bA,\bH,\bB} \left\|\dot{\bX} - \begin{bmatrix}\bA,~\bH, ~ \bB\end{bmatrix} \cD \right\|_F,
\end{equation}
where $\cD = {\scriptsize \begin{bmatrix} \bX \\ \bX\totimes \bX \\ \mathbf{1} \end{bmatrix} }$ with $\mathbf{1}$ being a row-vector of ones,  and the product $\totimes$ is defined as	$\bG\totimes \bG = \begin{bmatrix}		\bg_1\otimes \bg_1,\ldots, \bg_\cN\otimes \bg_\cN 	\end{bmatrix}$ with $\bg_i$ being $i$-th column of the matrix $\bG\in \R^{n\times \cN}$. Note that \[\bg\otimes \bg = \left[\bg_1^2, \bg_1\bg_2, \ldots, \bg_1\bg_n,\ldots, \bg_n^2\right]^\top \in \R^{n^2},\] where $\bg = [\bg_1,\ldots, \bg_n]^\top$. Additionally, the reader should notice that whenever the provided data is low-dimensional,  the projection onto the POD coordinates in equation \eqref{eq:proj_step} is not required. 

Although the optimization problem~\eqref{eq:opinf_optimization} appears straightforward, it has a couple of major drawbacks. Firstly,  it requires knowledge of the state derivative. When the collected data are finely sampled and noise-free, one can estimate it using numerical methods. However, in practical scenarios, data can be scarce and noisy, thus making the computation of the derivative information challenging. As a remedy to this, one can cast the problem \eqref{eq:opinf_optimization} in an integral form. Several works have discussed such an approach, see, e.g., \cite{kumpati1990identification,chen2018neural, goyal2022discovery,uy2022operator}. This approach enables robust learning of operators under conditions of limited and noisy data. We discuss it in more detail later in \Cref{sec:intergrationscheme}. 
In addition,  the matrix $\cD$ can be ill-conditioned, thus making the optimization problem \eqref{eq:opinf_optimization} challenging. A way to circumvent this problem is to make use of suitable regularization schemes, and many proposals are made in this direction in the literature, see, e.g., \cite{morYilGBetal20,mcquarrie2020data}.  

Another challenge---one of the most crucial ones, if not the most, from the dynamical system perspective---is regarding the stability of inferred models. When the optimization problem \eqref{eq:opinf_optimization} is solved, then the inferred operators only aim to minimize the specific design objective. However, it does not guarantee that the resulting dynamical system will be stable; Therefore, as major contributions of the paper, in the next sections, we propose novel inference problems that address this issue and ensure local and global asymptotic stability, or global bounded stability of the learned dynamical systems by construction. 

Lastly, it is worth noting that when $\by(t)$ is a low-dimensional, we do not need to identify a low-dimensional projection. We instead can learn a quadratic model using data for $\by$. Furthermore, to obtain interpretable models, we can employ the sparse-regression methodology \cite{brunton2016discovering}, where we seek to identify the operators $\bA$, $\bH$, and $\bB$ that are potentially sparse. 

\section{Inference of Locally Asymptotically Stable Quadratic Models}\label{sec:local_Stabil}
In this section, we present a characterization of locally asymptotically stable quadratic models of the form~\eqref{eq:quad_model}. This characterization allows us to parametrize the inference problem, which enables constructing locally asymptotically stable quadratic models through deliberate parameterization. 

\subsection{Local asymptotic stabilty characterization}
Considering quadratic systems with zero as a stable equilibrium point, we can use their linearization around zero to determine local asymptotic stability. The matrix $\bA$ in \eqref{eq:quad_model} defines the Jacobian of the system. 
It can be shown that the quadratic system \eqref{eq:quad_model} is locally asymptotically stable (\texttt{LAS}) if none of the eigenvalues of the matrix $\bA$ is in the closed right half complex plane.  
Hence, we only need to ensure the stability of the matrix $\bA$ for \las. 
In this regard, a recent work characterizes a stable matrix in \cite[Lemma~1]{gillis2017computing}, where every asymptotically stable matrix $\bA$ can be written as follows:
\begin{equation}\label{eq:par_stabil}
	\bA = (\bJ - \bR)\bQ  
\end{equation}
where $\bJ = -\bJ^{\top}$ is a skew-symmetric matrix, and $\bR = \bR^{\top}\succ  0 $ and $\bQ = \bQ^\top\succ 0$ are symmetric positive definite (SPD) matrices. Using this parameterization for stable matrices, the work \cite{morGoyPB23} learns stable linear systems. Moreover, since zero is a stable equilibrium point, this implies $\bB = 0$. Therefore, we assume $\bB$ to be zero in the inference of locally asymptotically stable models.
It is also worth noticing that the parametrization \eqref{eq:par_stabil} encodes a Lyapunov function for the underlying dynamical system, which is stated in the following lemma.

\begin{lemma}
	Consider a quadratic system as in \eqref{eq:quad_model}, where $\bA$ takes the form given in \eqref{eq:par_stabil}, i.e., $\bA = (\bJ -\bR)\bQ$, where $\bJ = -\bJ^\top$, and 
	$\bR, \bQ$ are SPD matrices. Then, the quadratic function
	\begin{equation*}
		\bV(\bx(t)) = \dfrac{1}{2}\bx(t)^{\top} \bQ\bx(t)  
	\end{equation*}
	is a Lyapunov function of the system when $\|\bx\|_2 <r$,  where
	\begin{equation}
		r = \dfrac{\sigma_{\min}(\bL)^2}{\|\bQ\|_2\|\bH\|_2},
	\end{equation}
	with $\bL\bL^\top =  \bQ^\top \bR\bQ$, and $\sigma_{\min}(\cdot)$ is the minimum singular value of a matrix. As a consequence, the system is locally asymptotically stable.
\end{lemma}
\begin{proof}
	First note that  $\bV(\bx) > 0$, for $\bx \neq 0$. Moreover, we have
	\begin{align*}
		\dot{\bV}(\bx(t)) &= \bx(t)^\top\bQ\dot\bx(t) = \bx(t)^\top\bQ^\top\dot\bx(t)\\ 
		&= \bx(t)^\top\bQ^\top \left(\bA \bx(t) + \bH\left(\bx(t)\otimes \bx(t)\right)\right)\\
		& = \bx(t)^\top\bQ^\top \left(\left(\bJ - \bR\right)\bQ \bx(t) + \bH\left(\bx(t)\otimes \bx(t)\right)\right)\\
		& = \cancelto{0}{\bx(t)^\top\bQ^\top \bJ \bQ\bx(t)} - \bx(t)^\top\bQ \left( \bR\bQ \bx(t) + \bH\left(\bx(t)\otimes \bx(t)\right)\right)\\
		& = - \bx(t)^\top\underbrace{\bQ^\top \bR\bQ}_{:= \bL \bL^{\top}} \bx(t) +  \bx(t)^\top\bQ^\top\bH\left(\bx(t)\otimes \bx(t)\right)\\
		& \leq -\sigma_{\min}(\bL)^2\|\bx(t)\|^2_2 + \|\bQ\|_2\|\bH\|_2\|\bx(t)\|^3_2.
	\end{align*}
Define $r = \dfrac{\sigma_{\min}^2(\bL)}{{\|\bQ\|_2\|\bH\|_2}}$. Then, $\dot{\bV}(\bx(t))  < 0$ for $\|\bx(t)\|_2 <r$ and $\bx(t) \neq 0$. As a consequence, $\bV(\bx)$ is a local Lyapunov function which proves the result. 
\end{proof}

Summarizing, the parameterization \eqref{eq:par_stabil} to construct the matrix $\bA$ characterizes \las~for quadratic systems. Thus, in contrast to previous attempts to learn stable models \cite{sawant2023physics}, we do not need to constrain the optimization problem by inequality constraints or spectral properties. Moreover, in contrast to \cite{sawant2023physics}, we can allow the matrix $A$ to be non-symmetric as well.

 we need not to impose constraints based on matrix inequality or eigenvalues to infer quadratic models. In contrast to \cite{sawant2023physics}, we can have the matrix $\bA$ non-symmetric as well. 

\subsection{Locally asymptotically stable model inference via suitable parameterization}\label{subsec:localstability_parameterization} Using the stable matrix parameterization, we state the problem formulation to infer \las~quadratic models as follows. Given the data $\bX$, e.g., in reduced-dimension, and the derivative information $\dot{\bX}$, we aim at inferring operators of a quadratic model using the following criterion:
\begin{equation}\label{eq:stable_learning_local}
\begin{aligned}
(\bJ, \bR, \bQ,\bH) &= \underset{\hat\bJ, \hat\bR,\hat \bQ, \hat\bH}{\arg\min}  \|\dot{\bX}-(\hat\bJ - \hat\bR)\hat\bQ\bX - \hat\bH\bX^\otimes\|_F,\\
& \qquad \text{subject to} ~~ \hat \bJ = -\hat\bJ^\top,  \hat\bR =  \hat\bR^\top \succ 0,  \hat\bQ =  \tilde\bQ^\top \succ 0.
\end{aligned}
\end{equation}
Once we have the optimal value for $\{\bJ,\bR,\bQ\}$, we can construct the stable matrix $\bA$ as $(\bJ - \bR)\bQ$, thus leading to a quadratic model of the form \eqref{eq:quad_model}.
Note that the optimization problem is a constrained one. However, as shown in, e.g., \cite{morGoyPB23}, we can use adequate parameterizations for these matrices leading to the unconstrained optimization problem
\begin{equation}\tag{\texttt{lasMI}}\label{eq:stable_learning_local1}
\begin{aligned}
(\bar\bJ, \bar\bR, \bar\bQ,\bar\bH) &= \underset{\acute\bJ, \acute\bQ, \acute\bR, \acute\bH}{\arg\min}  \|\dot{\bX}-(\acute\bJ -\acute\bJ^{\top} - \acute\bR\acute\bR^\top) \acute\bQ\acute\bQ^\top\bX - \acute\bH\bX^\otimes\|_F.
\end{aligned}
\end{equation}
The above formulation is obtained by utilizing the Cholesky factorization of $\bQ$ and $\bR$, and that any skew-symmetric matrix $\hat\bJ$ can be written as $\hat\bJ = \acute{\bJ} - \acute{\bJ}^\top$.
Then, we can construct the system \eqref{eq:quad_model} with a matrix $\bA$ with  $\left(\bar\bJ -\bar\bJ^{\top} - \bar\bR\bar\bR^\top\right) \bar\bQ\bar\bQ^\top$ and $\bH = \bar\bH$. Not surprisingly, the optimization problem \eqref{eq:stable_learning_local1} is nonlinear and non-convex; hence, there is no analytical solution to the problem. Therefore, we utilize a gradient-based approach to obtain a solution to the problem.

\section{Inference of Globally Asymptotically Stable Quadratic Models}\label{sec:GlobStab} 
Several quadratic dynamical systems are not only locally stable but also globally stable. We cannot guarantee their global stability properties if we enforce only local stability like in the previous section. Thus, in this section, we study the problem of enforcing global asymptotic stability to quadratic systems~\eqref{eq:quad_model}. To this aim, we first define the concept of energy-preserving nonlinearities, initially presented in \cite{lorenz1963deterministic, schlegel2015long}, which we, later on, generalize. 
This allows us to use the Lyapunov direct method to establish sufficient and necessary conditions for a  quadratic system with generalized energy-preserving nonlinearities to be globally asymptotically stable (\gas).
Finally, we propose a parametrization of quadratic systems that inherently has the global stability property, which is then leverage it to infer \gas quadratic models.

\subsection{Global asymptotic stability characterization}
Let us consider quadratic systems of the form \eqref{eq:quad_model}. In this section, we assume that $\bB = 0$, so the origin is an equilibrium of the system \eqref{eq:quad_model}. Additionally, we assume that the quadratic nonlinearity satisfies algebraic constraints, which are also discussed in \cite{schlegel2015long,lorenz1963deterministic}.
Precisely, the Hessian $\bH$ in \eqref{eq:quad_model} is said to be energy-preserving when it satisfies:
\begin{equation}\label{eq:energyPreserving_H_condition}
\bH_{ijk} + \bH_{ikj} + \bH_{jik} + \bH_{jki} + \bH_{kij} + \bH_{kji} = 0, 
\end{equation} 
where  $ \{i,j,k\} \in \{1, \dots n\}$, $\bH_{ijk} := e_i^{\top} \bH(e_j \otimes e_k)$, and $e_i \in \R^n$ denotes the $i$th unit vector. The condition \eqref{eq:energyPreserving_H_condition}, in terms of the Kronecker-product notation,  can be written as follows (see \Cref{theorem:EquivalKron} in \Cref{appendix}):
\begin{equation}\label{eq:EnergyPreservQuadTerm}
\bx^{\top}\bH(\bx\otimes \bx) = 0, \quad \text{for every $\bx \in \R^n$}.
\end{equation} 
For quadratic systems possessing energy-preserving Hessians, we can characterize monotonically \gas, i.e., to establish the conditions under which the energy of the state vector $\bE(\bx(t)) := \frac{1}{2} \bx^{\top}(t)\bx(t) = \frac{1}{2}\|\bx(t)\|_2^2$ is strictly monotonically decreasing for all trajectories of the quadratic system \eqref{eq:quad_model}.
\begin{theorem}\label{theorem:MonGlobStab} Consider a quadratic system \eqref{eq:quad_model} with  energy-preserving Hessian, meaning it satisfies \eqref{eq:EnergyPreservQuadTerm}. Then, the following statements are equivalent:
	\begin{enumerate}[(a)]
		\item The quadratic system \eqref{eq:quad_model} is monotonically \gas.
		\item The matrix $\bA_s = \frac{1}{2}\left(\bA +\bA^{\top}\right)$ is asymptotically stable, i.e., the eigenvalues of the matrix $\bA_s$ are all strictly negative real numbers.
	\end{enumerate}	
\end{theorem}
\begin{proof} $(b) \Rightarrow (a)$: Let us consider the state energy of the system as $\bE(\bx(t)) = \frac{1}{2} \bx^{\top}(t)\bx(t),$ where $\bx(t)$ is a trajectory of the considered quadratic system. The derivative of the energy function $\bE$ with respect to time $t$ at $\bx(t)$ is given by
	\begin{align*}
	\frac{d}{dt}\bE(\bx(t)) &= \dfrac{1}{2}\left(\frac{d}{dt}\bx(t)\right)^{\top}\bx(t) + \dfrac{1}{2} \bx(t)^{\top}\left(\frac{d}{dt}\bx(t)\right) \\
	&= \dfrac{1}{2}\left(\bA\bx(t) +  \bH(\bx(t)\otimes \bx(t))\right)^{\top}\bx(t) + \dfrac{1}{2} \bx(t)^{\top}\left(\bA\bx(t) +  \bH(\bx(t)\otimes \bx(t)\right) \\
	& =  \frac{1}{2}\bx(t)^{\top}(\bA+\bA^\top)\bx(t) +  \cancelto{0}{\bx(t)^{\top}\bH(\bx(t)\otimes \bx(t))} \\
	&  = \bx(t)^{\top}\bA_s\bx(t) <0.
	\end{align*}
	Hence, $\bE(\bx(t))$ is a strict Lyapunov function, and the quadratic system is, thus, monotonically \gas.	 

	$(a) \Rightarrow (b)$:  If the system in \eqref{eq:quad_model} is monotonically \gas, then $\frac{d}{dt}\bE(\bx(t)) <0 $. Moreover, if the system has an energy-preserving quadratic term, then $\frac{d}{dt}\bE(\bx(t)) = \bx(t)^{\top}\bA_s\bx(t)$. Hence, the condition $\bx(t)^{\top}\bA_s\bx(t) < 0$ must hold. As $\bA_s$ is symmetric, this implies that all the eigenvalues of $\bA_s$ are strictly negative.
\end{proof}
Notice that \Cref{theorem:MonGlobStab} characterizes monotonically \gas~for quadratic systems, provided that the quadratic term is energy-preserving.   However, it does not cover \gas~quadratic systems for which $\bE(\bx) = \|\bx\|_2^2$ is not a Lyapunov function. Indeed, a sufficient condition for a system to be \gas~is the existence of a global strict Lyapunov function. The following remark illustrates by means of an example that \Cref{theorem:MonGlobStab} is not enough to determine \gas in general. 
\begin{remark}  Let us consider the quadratic system of the form \eqref{eq:quad_model}, where
	\[\tilde{\bA} = \begin{bmatrix}
	-4 & -4 
	\\ 
	1  & 0
	\end{bmatrix}\quad \text{and}\quad \tilde{\bH} = \begin{bmatrix}
	2 	& 5 	& 4 	& 10 \\
	-1 	& -2 	& -2 	& -4
	\end{bmatrix}. \]	  
	Notice that the matrix $\frac{1}{2}(\tilde{\bA}+\tilde{\bA}^{\top})$ is not stable; additionally, the quadratic term represented by $\tilde{\bH}$ is not energy-preserving, meaning it does not satisfy \eqref{eq:EnergyPreservQuadTerm}. Hence, stability cannot be concluded from \Cref{theorem:MonGlobStab} for this example. However, the system is \gas~since 
	\[ \bV(\bx(t)) = \bx^{\top}(t)\tilde{\bQ}\bx(t) \quad \text{with} \quad  \tilde{\bQ} = \begin{bmatrix}
	1 & 2 \\ 
	2 & 5
	\end{bmatrix}\succ 0.  
	\]
	is a global strict Lyapunov function. 
\end{remark}
The above remark motivates us to derive more general quadratic Lyapunov functions for \gas~of the form, namely, $\bV(\bx(t)) = \bx^{\top}(t)\bQ\bx(t)$ for a given SPD matrix $\bQ$. 
To that aim, we introduce a notion of a generalized energy-preserving quadratic term with respect to $\bQ$.  For a given SPD matrix  $\bQ$,  the $\bH$ matrix in \eqref{eq:quad_model} is said to be generalized energy-preserving when
\begin{equation}\label{eq:GenEnergyPreservQuadTerm}
\bx^{\top}\bQ\bH(\bx\otimes \bx) = 0, \quad \text{for every $\bx \in \R^n$}.
\end{equation} 

Next, we present conditions under which the function $\bV(\bx(t)) = \bx^{\top}(t)\bQ \bx(t)$ would be a strict Lyapunov function for the underlying quadratic system. As a result,  due to the direct Lyapunov method, the quadratic system is \gas.

\begin{corollary}\label{col:GenGlobStab} Consider a quadratic system \eqref{eq:quad_model}, where $\bH$ is  generalized energy-preserving with respect to an SPD matrix $\bQ$. Then, the following statements are equivalent:
	\begin{enumerate}[(a)]
		\item $\bV(\bx(t)) = \bx^{\top}(t)\bQ\bx(t)$ is a global strict Lyapunov function for the quadratic system \eqref{eq:quad_model}.
		\item The matrix $(\bQ\bA)_s := \frac{1}{2}\left(\bQ\bA +\bA^{\top}\bQ\right)$ is strictly stable, i.e, the eigenvalues $(\bQ\bA)_s $ are all strictly negative real numbers.
	\end{enumerate}	
	Moreover, if these conditions hold, the quadratic system \eqref{eq:quad_model} is \gas.
\end{corollary}
\begin{proof} 
	Notice that the above result coincides with \Cref{theorem:MonGlobStab} if $\bQ = \bI$. Whenever $\bQ \neq \bI$, the result is a consequence of \Cref{theorem:MonGlobStab} when the linear change of coordinates as $\tilde{\bx} := \bQ^{\frac{1}{2}}\bx$ is applied to the quadratic system \eqref{eq:quad_model}.
\end{proof}
\Cref{col:GenGlobStab} provides a generalization of \Cref{theorem:MonGlobStab} for general quadratic Lyapunov functions. Additionally, it is worthwhile to stress that if $(\bQ\bA)_s$ is stable but not strictly, then $ \bV(\bx(t)) = \bx^{\top}(t)\bQ\bx(t)$ is still a global  Lyapunov function but not in a strict sense. Consequently, the system will still be globally stable but not necessarily asymptotic. In addition to this, when $(\bQ\bA)_s = 0$, then   $\bV(\bx(t))$ is constant for all the trajectories, i.e., the system \eqref{eq:quad_model} is Hamiltonian and energy-preserving. 

Based on this result, we focus on parametrizing a family of \gas~quadratic systems having a generalized energy-preserving term. This result is stated in the following lemma.
\begin{lemma}\label{lemma:global_stable}
	A quadratic system \eqref{eq:quad_model} represented by the matrices $\bA$ and $\bH$, where the matrix $\bH$ is generalized energy-preserving with respect to the SPD matrix $\bQ$, is \gas~ if 
	\begin{equation}\label{eq:str_A}
	\bA = (\bJ-\bR)\bQ,
	\end{equation}
	where $\bJ = -\bJ^\top$, $\bR $ is an SPD matrix, and 
	\begin{equation}\label{eq:str_H}
	\bH = \begin{bmatrix} \bH_1\bQ, \ldots, \bH_n\bQ \end{bmatrix},
	\end{equation}
	with $\bH_i \in \Rnn$ being skew-symmetric, i.e., $\bH_i = -\bH_i^\top$. 
	Moreover,  $\bV(\bx(t)) = \frac{1}{2}\bx^\top(t) \bQ\bx(t)$ is a global Lyapunov function for the underlying system. 
\end{lemma}
\begin{proof}
	Firstly, notice that $\bH = \begin{bmatrix} \bH_1\bQ, \ldots, \bH_n\bQ \end{bmatrix}$ is generalized energy-preserving. Thus, 
	\begin{align*}
	\bx^{\top}\bQ\bH(\bx\otimes \bx) 
	&=   \bx^{\top}\begin{bmatrix} \bQ\bH_1\bQ, \ldots, \bQ\bH_n\bQ \end{bmatrix}(\bx\otimes\bx)\\
	&=   \left(\sum_{i=1}^n\bx_i\left(\bx^{\top}\bQ\bH_i\bQ\bx \right)\right) = 0,
	\end{align*}
	since $\bQ\bH_i\bQ = \bQ\bH_i\bQ^\top$ are  skew-symmetric matrices for $i \in\{ 1, \dots n\}$. Hence, the parametrization~\eqref{eq:str_H} provides a generalized energy-preserving quadratic term. As a consequence, from \Cref{col:GenGlobStab}, $\bV(\bx(t)) = \bx^{\top}(t)\bQ\bx(t)$ is a global Lyapunov function if and only if $\frac{1}{2}\left(\bQ\bA +\bA^{\top}\bQ\right)$ is negative definite. Clearly, the parametrization $\bA = (\bJ-\bR)\bQ$ satisfies this condition, because 
	\[\frac{1}{2}\left(\bQ(\bJ-\bR)\bQ + \bQ(\bJ^\top-\bR)\bQ\right) = -\bQ\bR\bQ = -\bQ^\top \bR\bQ \prec 0,\] thus concluding the proof.
\end{proof}
Based on the discussions so far, we present a few key observations, which are as follows.
\begin{itemize}
	
	\item Every matrix $\bA$ for which $\frac{1}{2}\left(\bQ\bA +\bA^{\top}\bQ\right) <0$ can be written in the  form $\bA = (\bJ-\bR)\bQ$, where $\bJ$ is skew-symmetric and $\bR$ is SPD. Moreover, \Cref{thm:equvilence_between_H} and \Cref{lemma:equvilence_between_H_gen} prove that any matrix $\bH$, satisfying $\bx^{\top}\bQ\bH(\bx\otimes \bx)=0$, can be re-written as a matrix $\tilde \bH$ so that $\tilde \bH$ has the form as in \eqref{eq:str_H} and $\bH(\bx\otimes\bx) = \tilde\bH(\bx\otimes\bx)$.  

	\item It is worth noticing that the matrix $\bH$, satisfying \eqref{eq:str_H}, does not satisfy $\bH(e_i\otimes e_j) \neq \bH(e_j\otimes e_i)$, where $e_i$ and $e_j$ are canonical unit-vectors.  However, suppose one is interested in obtaining a representation where  $\bH(e_i\otimes e_j) = \bH(e_j\otimes e_i)$. In that case, we suggest performing a symmetrization trick as, for example, in \cite{morBenB15} once an inference or identification of a \gas~system is achieved. 
	
	\item  If the quadratic system \eqref{eq:quad_model} is globally stable but not asymptotically stable, then $\bA = (\bJ-\bR)\bQ$ with $\bR = \bR^{\top}\succeq 0$, instead of being a definite matrix. Hence, the parametrization that allows $\bR$ to be semi-definite also includes globally (non-asymptotic) stable systems.
\end{itemize}

The parametrization from \Cref{lemma:global_stable} guides us to cast the inference problem to obtain \gas~quadratic systems from data, which we discuss next.
\subsection{Global-stability informed learning}\label{subsec:Glob_stability_parameterization} Benefiting from 
\Cref{lemma:global_stable}, we can write down an inference problem to obtain a \gas~quadratic model using the corresponding $\bX$ and $\dot{\bX}$ data  (for the definitions of $\bX$ and $\dot\bX$, see \Cref{sec:overview})  as follows:
\begin{equation}\label{eq:stable_learning_global}
\begin{aligned}
(\bJ, \bR, \bQ,\bH_1,\ldots, \bH_n) &= \underset{\hat\bJ, \hat\bR,\hat\bQ, \hat\bH_1,\ldots,\hat\bH_n}{\arg\min}  \left\|\dot{\bX}-(\hat\bJ - \hat\bR)\hat\bQ\bX - \begin{bmatrix}\hat\bH_1\hat\bQ,\ldots,\hat\bH_n\hat\bQ\end{bmatrix}\bX^\otimes\right\|_F,\\
& \qquad \text{subject to} ~~ \hat\bJ = -\hat\bJ^\top,  \hat\bR =  \hat\bR^\top \succ 0,  \hat\bQ =  \hat\bQ^\top \succ 0,~\text{and}\\
&\hspace{2.5cm}\hat{\bH}_i = -\hat{\bH}^{\top}_i,~i\in\{1,\ldots,n\}.
\end{aligned}
\end{equation}
Having the optimal tuple $(\bJ, \bR, \bQ,\bH_1,\ldots, \bH_n)$ solving \eqref{eq:stable_learning_global}, we can construct the matrices $\bA$ and $\bH$ as follows:
\begin{equation}
\bA = (\bJ-\bR)\bQ,\qquad \bH = \begin{bmatrix}\bH_1\bQ,\ldots, \bH_n\bQ \end{bmatrix},
\end{equation}
thus leading to a quadratic model of the form \eqref{eq:quad_model}, which is \gas. Similar to the case \eqref{eq:stable_learning_local}, the above problem enforces constraints on the matrices. To remove these constraints, we can parameterize analogously. The parameterization of the matrices $\bJ,\bR$, and $\bQ$ remains the same as in Subsection \ref{subsec:localstability_parameterization}. We can parameterize $\bH_i$ analogous to $\bJ$, but this might be computationally inefficient. However, there is a computationally efficient way to parameterize $\bH$, satisfying \eqref{eq:str_H}. For this, let us consider a tensor $\cH \in \R^{n\times n\times n}$, and denotes its mode-1 and mode-2 matricizations by $\cH^{(1)}$ and $\cH^{(2)}$, respectively. For the definition of the matricization for tensors, we refer to \cite{kolda2009tensor}. Then, we construct the matrix $\bH$ as follows:
\[\left(\cH^{(1)} - \cH^{(2)}\right)\left(\bI \otimes \bQ\right),\] then it fulfills the condition \eqref{eq:str_H} by construction. Intuitively, matrices obtained using mode-1 and mode-2  matricizations can be interpreted as transposes of each other in a particular way; hence, $\cH^{(1)} - \cH^{(2)}$ would yield skew-symmetric matrices in the frontal slices, satisfying the condition \eqref{eq:str_H} whenever $\bQ = \bI$. Thus, we can recast the constraint optimization problem \eqref{eq:stable_learning_global} as an unconstrained one as follows:
\begin{equation}\tag{\texttt{gasMI}}\label{eq:stable_learning_global1}
\begin{aligned}
(\bar\bJ, \bar\bR, \bar\bQ,\bar\cH) &= \underset{\acute\bJ, \acute\bQ, \acute\bR, \acute\cH}{\arg\min}  \left\|\dot{\bX}-(\acute\bJ -\acute\bJ^{\top} - \acute\bR\acute\bR^\top) \acute\bQ\acute\bQ^\top\bX - \left(\acute\cH^{(1)}- \acute\cH^{(2)}\right)\left(\bI\otimes\acute\bQ\acute\bQ^\top\right) \bX^\otimes\right\|_F.
\end{aligned}
\end{equation}
This will then allow us to construct the system \eqref{eq:quad_model} with matrices $\bA$ and $\bH$ as 
\begin{equation*}
\begin{aligned}
\bA &= \left(\bar\bJ -\bar\bJ^{\top} - \bar\bR\bar\bR^\top\right) \bar\bQ\bar\bQ^\top, ~~\text{and}~~
\bH &= \left(\bar\cH^{(1)}- \bar\cH^{(2)}\right)\left(\bI\otimes\bar\bQ\bar\bQ^\top\right).
\end{aligned}
\end{equation*}

\section{Quadratic Models with attracting trapping regions}\label{sec:atr}
Until now, we have discussed the stability of quadratic systems from the perspective of asymptotic stability. However, as we mentioned in the introduction, nonlinear dynamical systems might possess attracting trapping regions, even when no fixed point, if it exists, inside this region is a stable equilibrium. One example of this phenomenon can be given by the widely-known Lorenz system \cite{lorenz1963deterministic}. Although for such systems, the long-term behavior of the trajectories is bounded, they do not converge asymptotically to a fixed point. In order to study such attracting trapping regions, the authors in \cite{schlegel2015long}, inspired by the results in \cite{lorenz1963deterministic} and the reformulation of the Lyapunov direct method in \cite{swinnerton2000note}, have characterized these regions as monotonic asymptotic attracting trapping regions when the quadratic systems \eqref{eq:quad_model} ensure energy-preserving nonlinearities. Building upon them, we first generalize these results by means of the use of a more general quadratic Lyapunov function of the form $\bV(\bx(t)) = \bx^{\top}(t)\bQ\bx(t)$ and the generalized energy-preserving quadratic terms as in \eqref{eq:GenEnergyPreservQuadTerm}.  Finally, we propose a parametrization for quadratic systems that exhibit attracting trapping regions (\atr). We then exploit this parametrization to infer \atr~quadratic models by construction.
\subsection{Boundedness and attracting trapping region characterization}
Consider a quadratic system as described in \eqref{eq:quad_model} with a generalized energy-preserving term with respect to an SPD matrix $\bQ$. 
A trapping region $\cM \subset \R^n$  is (globally) asymptotically attracting if there exists a Lyapunov function, which strictly monotonically decreases along all trajectories starting from an
arbitrary state outside of $\cM$. 
This implies that outside of the trapping region, all trajectories should cross its border in a finite amount of time. To characterize such behavior, let us consider a closed ball $\cB(\bm,r)$ centered at $\bm \in\R^n$ with radius $r>0$ that contains the trapping region $\cM$. Consequently, a strict Lyapunov function exists outside the ball $\cB(\bm,r)$. Having said that, we consider the state translation $\tilde{\bx}(t) = \bx(t) -\bm$ applied to the system  \eqref{eq:quad_model}, which leads to the translated quadratic system
\begin{align}
	\dot{\tilde{\bx}}(t) &= \bA\left(\tilde{\bx}(t) +\bm\right) +  \bH\left((\tilde{\bx}(t) +\bm)\otimes (\tilde{\bx}(t) +\bm)\right) + \bB, \nonumber\\
	&= \tilde{\bA}\tilde{\bx}(t) +  \tilde{\bH}\left(\tilde{\bx}(t)\otimes \tilde{\bx}(t)\right) +\tilde{\bB},	\label{eq:Qsys_trans}
\end{align}
where $\tilde{\bA} = \bA + \bH(\bI \otimes \bm) + \bH(\bm \otimes\bI) $, $\tilde{\bH} = \bH$ and $\tilde{\bB} = \bB + \bH(\bm \otimes \bm) + \bA\bm$. Notice that the translated system \eqref{eq:Qsys_trans} possesses the forcing term $\tilde{\bB}$, which means that the origin would not be an equilibrium point. For the translated system, we set a quadratic Lyapunov function as
$\bV(\tilde{\bx}(t)) = \frac{1}{2}\tilde{\bx}(t)^{\top}\bQ \tilde{\bx}(t)$, which can be used to characterize an \atr~for quadratic systems. The result below is a generalization of the one from \cite[Theorem 1]{schlegel2015long} for the case of general quadratic Lyapunov functions.  
\begin{corollary}\label{cor:atr_char} Consider a quadratic system \eqref{eq:Qsys_trans}, where the matrix $\bH$ is generalized energy-preserving with respect to an SPD matrix $\bQ$. Then, the following statements are equivalent:
	\begin{enumerate}[(a)]
		\item There exist a radius $r>0$ and a shifting vector $\bm \in \R^n$ for which $\bV(\tilde{\bx}(t)) = \tilde{\bx}^{\top}(t)\bQ\tilde{\bx}(t)$ is a strictly global Lyapunov function for the system outside the ball $\cB(\bm,r)$. 
		\item The matrix $(\bQ\tilde{\bA})_s := \frac{1}{2}\left(\bQ\tilde{\bA} +\tilde{\bA}^{\top}\bQ\right)$ is asymptotically stable, i.e, the eigenvalues $(\bQ\tilde\bA)_s $ are all strictly negative real numbers, where $\tilde{\bA} = \bA + \bH(\bI \otimes \bm) + \bH(\bm \otimes\bI)$.
	\end{enumerate}		
	Moreover, when these conditions hold, the ball $\cB(\bm,r)$ contains an asymptotic trapping region of the quadratic system \eqref{eq:quad_model}, where $r = \frac{1}{\sigma_{\min}((\bQ\bA)_s)} \|\tilde{\bB}\|.$
\end{corollary}    
\begin{proof}
	This proof follows the steps in \cite{schlegel2015long} for the general case of the Lyapunov function $\bV(\tilde{\bx}(t)) = \tfrac{1}{2}\tilde{\bx}(t)^{\top}\bQ \tilde{\bx}(t)$. The derivative of $\bV(\bx(t))$ is given as
	\begin{align*}\allowdisplaybreaks
		\frac{d}{dt}\bV(\tilde{\bx}(t)) 
		&= \frac{1}{2}\left(\frac{d}{dt}\tilde{\bx}(t)\right) ^{\top}\bQ\bx(t) + \frac{1}{2}\tilde{\bx}(t)^{\top}\bQ\left(\frac{d}{dt}\tilde{\bx}(t)\right)  
		\\
		& = \frac{1}{2}\left(\tilde{\bA}\tilde{\bx}(t) +  \tilde{\bH}\left(\tilde{\bx}(t)\otimes \tilde{\bx}(t)\right) +\tilde{\bB}\right) ^{\top}\bQ\bx(t) + \frac{1}{2}\tilde{\bx}(t)^{\top}\bQ\left(\tilde{\bA}\tilde{\bx}(t) +  \tilde{\bH}\left(\tilde{\bx}(t)\otimes \tilde{\bx}(t)\right) +\tilde{\bB}\right)  \\
		&= \frac{1}{2}\bx(t)^{\top}\left(\tilde\bA^{\top}\bQ +\bQ\tilde\bA\right){\bx}(t) + \cancelto{0}{ \tilde{\bx}(t)^{\top}\bQ\bH\left(\tilde{\bx}(t)\otimes \tilde{\bx}(t)\right)} +\tilde \bx(t)^{\top}\tilde{\bB}\\
		&=  \tilde\bx(t)^{\top}(\bQ\tilde{\bA})_{s}\tilde\bx(t) + \tilde{\bx}^{\top}(t)\tilde{\bB}.
	\end{align*}
	Notice that for large enough $\tilde{\bx}(t)$  (e.g., in 2-norm), the quadratic term $\tilde\bx(t)^{\top}(\bQ\bA)_s\tilde\bx(t)$ dominates the linear term. Hence,  $\tilde\bx(t)^{\top}\bA\tilde\bx(t) <0$ if and only if there exist a radius $r>0$ for which $\bV(\tilde{\bx}(t))$ is a Lyapunov function outside the ball $\cB(\bm,r)$. 
	
	In order to estimate the radius $r$, notice that
	\begin{align*} 	\frac{d}{dt}\bV(\tilde{\bx}(t))  &=  \tilde\bx(t)^{\top}(\bQ\tilde{\bA})_{s}\tilde\bx(t) + \tilde{\bx}^{\top}(t)\tilde{\bB}\\
		&\leq -\sigma_{\min}{(\bQ\tilde{\bA})_{s}}\|\bx(t)\|^2  + \|\tilde{\bB}\|\|\bx(t)\|.
	\end{align*}
	Hence, if $\|\bx(t)\| > \frac{1}{\sigma_{\min}{(\bQ\tilde{\bA})_{s}}}\|\tilde{\bB}\| :=r$, then $\dot\bV(\tilde{\bx}(t)) < 0$. Hence, the ball $\cB(\bm,r)$ is an attracting trapping region.
\end{proof}	
\Cref{cor:atr_char} provides a characterization for \atr~systems, by means of the shifted system \eqref{eq:Qsys_trans}.  
It enables us to parametrize a family of quadratic systems with generalized energy-preserving nonlinearities that has \atr.

\begin{lemma}\label{lemma:par_atrapping_regions} 	A quadratic system \eqref{eq:quad_model}, represented by the matrices $\bA$ and $\bH$, where $\bH$ is generalized energy-preserving with respect to an SPD  matrix $\bQ$,  is \atr~ if there exists an $\bm \in \R^n$ such that matrices $\tilde{\bA}$ and $\tilde{\bH}$ of the translated system \eqref{eq:Qsys_trans}  are given by
	\begin{equation}\label{eq:str_Atr}
	\tilde{\bA} = (\bJ-\bR)\bQ,
	\end{equation}
	where $\bJ = -\bJ^\top$, $\bR = \bR^\top \succ 0$, and $\bQ = \bQ^\top \succ 0$, and  
	\begin{equation}\label{eq:str_Htr}
	\tilde{\bH} = \begin{bmatrix} \bH_1\bQ, \ldots, \bH_n\bQ \end{bmatrix},
	\end{equation}
	where $\bH_i \in \Rnn$ with $\bH_i = -\bH_i^\top$.
	Moreover,  $\bV(\tilde{\bx}(t)) = \frac{1}{2}\tilde{\bx}^\top(t) \bQ\tilde{\bx}(t)$ with $\tilde\bx(t) = \bx(t) - \bm $ is a global Lyapunov function for the underlying system outside the ball $\cB(\bm,r)$, with $r = \frac{1}{\sigma_{\min}((\bQ\bR)_s)} \|\tilde{\bB}\|.$ 
\end{lemma}    
\begin{proof} The proof follows the same line of thoughts as \Cref{lemma:global_stable}. We, therefore, skip it for brevity of the paper.
\end{proof}
\Cref{lemma:par_atrapping_regions} provides a parametrization for \atr~systems, guiding us to cast the inference problem, which we discuss next. As for globally \atr~systems, we briefly comment on the case where $\bR$ is semi-definite. If $\bR$ is allowed to be semi-definite, then $\bV(\tilde{\bx}(t)) = \tilde{\bx}^{\top}(t)\bQ\tilde{\bx}(t)$ will be a non-strictly global Lyapunov function if $\bv^{\top}\tilde{\bB} = 0$, where $\tilde\bB$ is defined in \eqref{eq:Qsys_trans},  for every $\bv\in\R^n$ satisfying $\bR\bQ\bv = 0$. Moreover, when $\bR =0$, the system is energy-preserving  if and only if $\tilde{\bB} = 0$, and  $\bV(\tilde{\bx}(t))$ is constant.

\subsection{Attracting trapping regions informed learning}
In this section, we propose a methodology to guarantee a learned quadratic system as in \eqref{eq:quad_model} to be \atr. To this aim, we notice that we can rewrite the dynamics \eqref{eq:quad_model}  in terms of the matrices of the shifted system in \eqref{eq:Qsys_trans}. Indeed, since $\tilde{\bx}(t) = \bx(t) -\bm$, we have
\begin{equation}
	\dot{\bx}(t)= \tilde{\bA}(\bx(t) -\bm) +  \tilde{\bH}\left(\left(\bx(t) -\bm\right)\otimes \left(\bx(t) -\bm\right)\right) +\tilde{\bB},
\end{equation}
where $\tilde \bA$ is a stable matrix, and $\tilde\bH$ is energy-preserving. 
Then, we can leverage the parametrization from \Cref{lemma:par_atrapping_regions} to obtain an \atr~system. Since the vector $\bm$ is also unknown, it needs to be a part of the inference problem as well. As a consequence, we can write down the inference problem to obtain \atr~ systems using the corresponding data as follows:
\begin{equation}\label{eq:stable_learning_atr}
	\begin{aligned}
		&(\tilde \bJ, \tilde\bR, \tilde\bQ,\tilde\bH_1,\ldots, \tilde\bH_n,\tilde\bB, \bm) \\
		&\qquad\qquad = \underset{ \hat \bJ, \hat\bR,\hat \bQ, \hat\bH_1,\ldots,\hat\bH_n, \hat{\bB},\hat{\bm}}{\arg\min}  
		\left\|\dot{\bX}- (\hat \bJ- \hat \bR)\hat\bQ(\bX -\hat{\bm}) - \begin{bmatrix} \hat\bH_1\hat\bQ, \ldots, \hat\bH_n\hat\bQ \end{bmatrix}\left(\bX-\hat{\bm}\right)^{\otimes} +\hat{\bB}\right\|
		\\
		& \qquad\qquad\qquad \text{subject to} ~~ \hat \bJ = -\hat\bJ^\top,  \hat\bR =  \hat\bR^\top \succ 0,  \hat\bQ =  \hat\bQ^\top \succ 0,~\text{and}\\
		&\qquad\qquad\hspace{2.5cm}\hat\bH_i = -\hat\bH_i^{\top},~i\in\{1,\ldots,n\}, \hat{\bB}, \hat{\bm} \in \R^n.
	\end{aligned}
\end{equation}
Having the optimal tuple $(\tilde \bJ, \tilde\bR, \tilde\bQ,\tilde\bH_1,\ldots, \tilde\bH_n,\tilde\bB,\bm)$ for \eqref{eq:stable_learning_atr}, we can construct $\bA$, $\bH$ and $\bB$ as follows:
\begin{equation}
	\bA = (\tilde\bJ-\tilde\bR)\tilde\bQ - \bH(\bI \otimes \tilde\bm) - \bH(\bm \otimes\bI),  \bB = \tilde\bB- \bH(\bm \otimes \bm) - \bA\bm, ~~\text{and}~~\bH = \begin{bmatrix}\tilde\bH_1\tilde\bQ,\ldots, \tilde\bH_n\tilde \bQ \end{bmatrix},
\end{equation}
thus leading to a quadratic model of the form \eqref{eq:quad_model} that is \atr.  Moreover, the same parametrization tricks as used in \Cref{subsec:Glob_stability_parameterization} for $\bJ,\bR$, $\bQ$ and $\bH$, can be employed here. Thus, we can recast the constrained optimization problem \eqref{eq:stable_learning_atr} as an unconstrained one as follows: \begin{equation}\tag{\texttt{atrMI}}\label{eq:stable_learning_atr1}
	\begin{aligned}
		(\bm,\bar\bJ, \bar\bR, \bar\bQ,\bar\cH) &= \underset{\bm,\acute\bJ, \acute\bQ, \acute\bR, \acute\cH}{\arg\min}  \left\|\dot{\bX}-(\acute\bJ -\acute\bJ^{\top} - \acute\bR\acute\bR^\top) \acute\bQ\acute\bQ^\top(\bX-\bm) \right.\\
		&\left.\qquad\qquad\qquad- \left(\acute\cH^{(1)}- \acute\cH^{(2)}\right)\left(\bI\otimes\acute\bQ\acute\bQ^\top\right) (\bX-\bm)^\otimes - \tilde{\bB}\right\|_F.
	\end{aligned}
\end{equation}
This then allows the construction of the system \eqref{eq:quad_model} with matrices  $\bA$  and $\bH$ given by
\begin{equation}
	\begin{aligned}
		\bH &= \left(\bar\cH^{(1)}- \bar\cH^{(2)}\right)\left(\bI\otimes\bar\bQ\bar\bQ^\top\right),\\
		\bA &= \left(\bar\bJ -\bar\bJ^{\top} - \bar\bR\bar\bR^\top\right) \bar\bQ\bar\bQ^\top - \bH(\bI \otimes \bm) - \bH(\bm \otimes\bI).\\
	\end{aligned}
\end{equation}
Once again, the inferred system is guaranteed to be \atr~due to its dedicated parameterization.  

\section{Optimizing Through a Numerical Integration Scheme to Infer Operators}\label{sec:intergrationscheme}
To infer operators using \ref{eq:stable_learning_local1}, \ref{eq:stable_learning_global1}, and \ref{eq:stable_learning_atr1}, we require the derivative information. However, obtaining this information can be challenging using numerical methods, particularly when the data are noisy, scarce, or irregularly sampled. Therefore, it would be beneficial to infer operators via an integral form of the differential equations, a technique that is widely adopted for learning continuous-time models \cite{chen2018neural, goyal2022discovery,uy2022operator}. For this, let us consider the following general form of a differential equation:
$$\dot{\bx}(t) = \bg_{\boldsymbol{\alpha}}(\bx(t)),$$
where $\bg$ defines a vector field and is parameterized by learnable ${\boldsymbol{\alpha}}$. In our case, $\bg$ has a quadratic form, and ${\boldsymbol{\alpha}}$ contains all learnable operators such as $\bJ$ and $\bR$.

Given the data $\bx(t)$ and $\dot{\bx}(t)$ at $t\in\{t_1,\ldots,\cN\}$, we can cast the learning of $\mathbf{g}$ as follows:
\[\min_{\boldsymbol{\alpha}}\sum_i\|\dot\bx(t_i) - \bg_{\boldsymbol{\alpha}}(\bx(t_i))\|.\]
We can also re-cast the above optimization problem in an integral form as follows:
\begin{equation}\label{eq:integral_form}
	\min_{\boldsymbol{\alpha}}\sum_i\left\|\bx(t_{i+1}) - \int_{t_i}^{t_{i+1}}\bg_{\boldsymbol{\alpha}}(\bx(t))\dt\right\|,
\end{equation}
which then avoids the need for derivative information. However, it comes at the expense that the problem \eqref{eq:integral_form} is more involved. The concepts of NeuralODEs \cite{chen2018neural} are developed, allowing us to solve such problems efficiently by taking a gradient of the learnable parameters through any numerical integration method with arbitrary accuracy. Despite this method being memory efficient $(\cO(1))$, it might invoke several function calls (in this case, the function $\bg$). This might add to computational costs. Therefore, at the expense of memory and accuracy, we approximate the integral in \eqref{eq:integral_form} using a fourth-order Runge-Kutta method as follows:
\begin{equation}\label{eq:integration}
\bx(t_{i+1}) = \Phi_{\dt}\left(\bg_{\boldsymbol{\alpha}}(\bx(t)), \bx(t_i\right) := \bx(t_i) + \dt\left(\bh_1 +\bh_2 + \bh_3 + \bh_4\right),
\end{equation}
where $\dt = t_{i+1} - t_i$ and
\begin{equation}
\begin{aligned}
\bh_1 &= \tfrac{1}{6}\bg_{\boldsymbol{\alpha}}(\bx(t_i)), & \bh_2 &= \tfrac{1}{3}\bg_{\boldsymbol{\alpha}}\left(\bx(\bt_i) + \dfrac{\dt}{2}\bh_1\right) \\
\bh_3 &= \tfrac{1}{3}\bg_{\boldsymbol{\alpha}}\left(\bx(t_i) + \tfrac{\dt}{2}\bh_2\right) & h_3 &= \dfrac{1}{6}\bg_{\boldsymbol{\alpha}}\left(\bx(t_i) + \dt \bh_3\right).
\end{aligned}
\end{equation}
Subsequently, we can write the objective function \eqref{eq:integral_form} as follows:
\begin{equation}\label{eq:Loss_function_RK}
\min_{\bar{\boldsymbol{\alpha}}} \sum_i\left\|{\bx}(t_{i+1})- \Phi_{\texttt{\dt}}\left(\bg_{\bar{\boldsymbol{\alpha}}}, \bx(t_i)\right)\right\|_F.
\end{equation}
We can utilize the above formation in our inference problems. For example, we can write the \ref{eq:stable_learning_local1} problem as follows:
\begin{equation}\label{eq:stable_learning_local_integral}
\begin{aligned}
\bar{\boldsymbol{\alpha}} &= \underset{\acute{\boldsymbol{\alpha}}}{\arg\min}  \sum_{i}\|{\bx}(t_{i+1})- \Phi_{\texttt{\dt}}\left(\tilde\bg_{\acute{\boldsymbol{\alpha}}}, \bx(t_i)\right)\|_F,
\end{aligned}
\end{equation}
where $\acute{\boldsymbol{\alpha}} = (\acute\bJ, \acute\bR, \acute\bQ,\acute\bH)$ and $\tilde\bg_{\acute{\boldsymbol{\alpha}}}(\bx) = (\acute\bJ - \acute\bJ^\top - \acute\bR\acute\bR^\top)\acute\bQ\acute\bQ^\top \bx + \acute\bH(\bx\otimes \bx) $.  This way, we can combine a numerical integration scheme with \ref{eq:stable_learning_local1} to infer the operators, thus avoiding the need of the derivative information at any stage. This can be done analogously for \ref{eq:stable_learning_global1} and \ref{eq:stable_learning_atr1} as well.

\section{Numerical Experiments}\label{sec:Exp}
In this section, we demonstrate the effectiveness of the proposed stability-guaranteeing methodologies, namely \ref{eq:stable_learning_local1}, \ref{eq:stable_learning_global1}, and \ref{eq:stable_learning_atr1}. 
To evaluate the performance of these methodologies, we apply them to two case studies involving high-dimensional data obtained from Burgers' equation and Chafee-Infante equations. Additionally, we use two examples, namely the chaotic Lorenz model and the magneto-hydrodynamics model, to demonstrate the utility of the proposed methodologies for discovering governing equations via sparse regression and learning energy-preserving models. We shall discuss a precise set-up of each example in their respective subsections. However, note that all three proposed methodologies incorporate a numerical integration scheme as discussed in \Cref{sec:intergrationscheme} so that we are not required to estimate the derivative information from data that could be challenging if they are scarce and noisy. This also enhances the performance of the classical \opinf~as well, as discussed in \cite{uy2022operator}. 

\paragraph{Training set-up:} Given that the involved optimization problems  in inferring (sparse) operators are nonlinear, non-convex, and lack analytical solutions, we adopt a gradient descent method for identification.
Specifically, we utilize the Adam optimizer \cite{kingma2014adam} with a triangular cyclic learning rate ranging from $10^{-6}$ to $10^{-2}$ over a cycle of $4~000$ steps and make $12~000$ updates in total. The coefficients of all matrices are initialized with random values drawn from a Gaussian distribution with a mean of $0$ and a standard deviation of $0.1$.
To evaluate the performance of our proposed methodology, we compare it with the operator inference method discussed in \cite{uy2022operator}, which utilizes the integral form of differential equations to learn operators but without imposing any constraints. We refer to it as \opinfbenchmark. To ensure a fair comparison, we also use the Runge-Kutta integration scheme as discussed in \Cref{sec:intergrationscheme} for \opinfbenchmark.
Furthermore, for sparse regression problems to discover governing equations, we employ a sequential thresholding algorithm \cite{brunton2016discovering,goyal2022discovery}. We set the tolerance $0.1$ and perform hard-pruning of coefficients below the tolerance after each $12~000$ updates. We optimize the remaining non-zero coefficients. We repeat this process four times to obtain a sparse parsimonious model which exhibits the desired stability properties. 
We implement all the methods using PyTorch on a $12$th Gen \intel~\coreifive-12600K 32GB RAM.

\paragraph{Setting $\bQ=\bI$ for \gasmi~and \atrmi:} 
The involvement of matrix $\bQ$ in defining both matrices $\bA$ and $\bH$ (see \Cref{lemma:global_stable,lemma:par_atrapping_regions}) makes the corresponding optimization problem even more challenging to solve.
To ease the problem, we set $\bQ$ to identity. It results in monotonic Lyapunov functions for \gas~and \atr. However, it can indeed limit the flexibility of the learned operators. We had attempted to learn $\bQ$ along with other parameters but encountered challenges with convergence and got stuck in bad local minima. In future work, we plan to investigate more efficient ways of solving the optimization problem when $\bQ \neq \bI$. This will require developing novel strategies to effectively handle the additional complexity introduced by the matrix $\bQ$. 
\subsection{Learning reduced operators for high-dimensional data via \opinf}
We begin by considering high-dimensional data and learning the corresponding operators in a low-dimensional state-space.
\subsubsection{Burgers' equations with Dirichlet boundary conditions}\label{sec:burgers_diri}
In our first example, we consider the one-dimensional Burgers' equation, which is governed by the following equations:
\begin{equation}\label{eq:burgers}
\begin{aligned}
v_t +  v v_\zeta&=  \mu \cdot v_{\zeta\zeta} & & \text{in}~(0,1)\times (0,T), \\
v(0,\cdot)  & = 0, ~~\text{and~~}
v(1,\cdot) = 0, &&\\
v(\zeta,0) &= v_0(\zeta) & & \text{in}~(0,1),
\end{aligned}
\end{equation}
where $v_t$ and $v_{\zeta}$ denote the derivative of $v$ with respect to the time $t$ and spatial coordinates $\zeta$, and $v_{\zeta\zeta}$ denotes the second derivative of $v$ with respect to $\zeta$. We set $\mu = 0.05$. We discretize using a central finite difference scheme. We consider $250$ equidistant grid points in the spatial domain, and $500$ equidistant data points are collected in the time-interval $[0,1]$. Furthermore, as in \cite{morGoyPB23}, we collect data using $17$ different initial conditions, i.e., 
$$v_0(\zeta) = \sin(f\cdot 2\pi \zeta) \zeta(1-\zeta), \qquad f = \{3, 3.125,\ldots, 4.875,5.0\}.$$
Taking data corresponding to three initial conditions ($f = \{3.5, 4.0, 4.5\}$) out for testing, we use the rest of them for learning the operators of a quadratic system as in \eqref{eq:quad_model}.

First, we plot the decay of the singular values of the training data in \Cref{fig:burgers_svd}. 
This serves as an indicator of the number of dominant modes that effectively capture the amount of energy present in the dataset. We notice a rapid decay of the singular values, thus a possibility of constructing low-order models yet capturing the underlying dynamics accurately. 
\begin{figure}[tb]
	\centering
	\includegraphics[width = 0.5\textwidth]{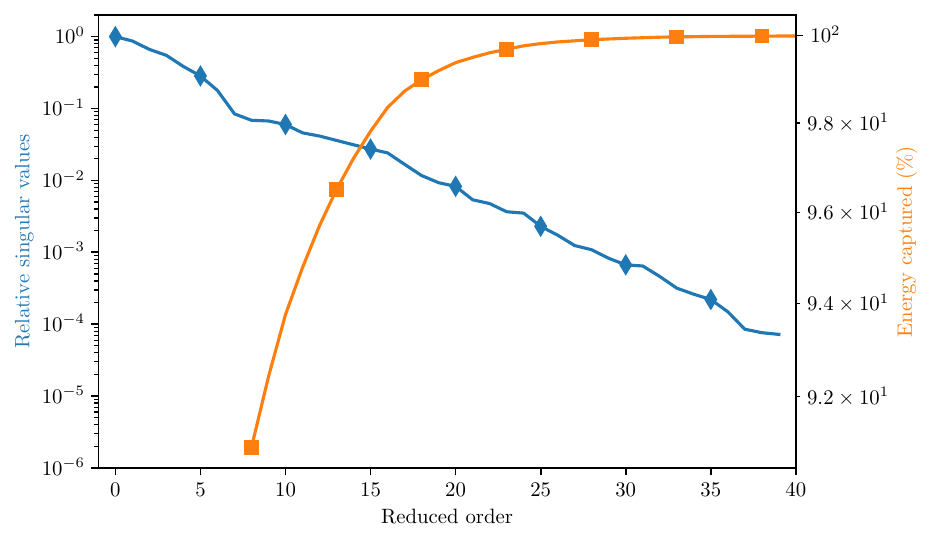}
	\caption{Burgers' equation  with Dirichlet boundary conditions: Decay of the singular values obtained using the training data. The orange graph indicates how much energy is captured by how many dominant modes. }
	\label{fig:burgers_svd}
\end{figure} 

In our first study, we aim to solve the inference problem using \opinfbenchmark, \lasmi, and \gasmi\ to obtain the corresponding dynamical systems. It is known that operator inference problems can be ill-conditioned; hence, they may require a regularization, which 
penalizes the norm of the nonlinear terms, namely the matrix $\bH$, in the optimization problem. In this study, along with the corresponding optimization problem for  \opinfbenchmark, \lasmi, and \gasmi, we add a regularizer based on the mean $l_1$-norm of the matrix $\bH$, i.e., $\lambda_\bH \dfrac{1}{m}\|\bH\|_{l_1}$, where $m$ is the number of the elements in the matrix $\bH$. A suitable choice of $\lambda_\bH$ is crucial for learning good quality data-driven models, particularly for \opinfbenchmark. Thus, to investigate the performance of these considered methods to learn operators, we set the order of the inferred model to $20$ and vary $\lambda_\bH$. We then test the quality of the learned models on the left-out testing data, which is discussed in the following.

We first project the left-out testing initial conditions on the corresponding low-dimensional subspace of order $20$ determined using the dominant subspace of the training data. We then simulate the reduced-order model in the time interval $[0,1]$ by considering $501$ points. We then re-project the solutions on the high-dimensional state-space using the same projection matrix. To measure the quality of the solution, we compute the relative $L_2$-norm of the error between the ground truth and the predicted solution using the learned models as follows:
\begin{equation}\label{eq:l2_measure_test}
\dfrac{\|\bX^{\texttt{ground-truth}} - \bX^{\texttt{learned}}\|_2}{\|\bX^{\texttt{ground-truth}}\|_2},
\end{equation}
where $\bX^{\texttt{ground-truth}}$ and $\bX^{\texttt{learned}}$, respectively, store the solutions using the high-fidelity models and using the inferred reduced quadratic models for a considered initial condition. The results for three test cases are plotted in \Cref{fig:Burgers_Dirichilet_time_domain}, where we notice that \opinfbenchmark~is quite sensitive to the parameter $\lambda_H$ and yields even unstable models, particularly for smaller $\lambda_H$. The method \lasmi~ yields also unstable models for various $\lambda_H$, but less often than \opinfbenchmark. On the other hand, \gasmi~ provides globally stable models for values of $\lambda_\bH$ as the stability property is encoded directly in its parameterization. 
However, for the value  $\lambda_\bH = 10^{-3}$, all these methods tend to  perform well.
For this value of $\lambda_\bH$, we present the solutions obtained using these learned models for one of the test cases on the full grid in \Cref{fig:burgers_time_domain_onetraj}, illustrating that the inferred model captures the essential dynamics.

\begin{figure}[tb]
	\centering
	\includegraphics[width = 0.95\textwidth]{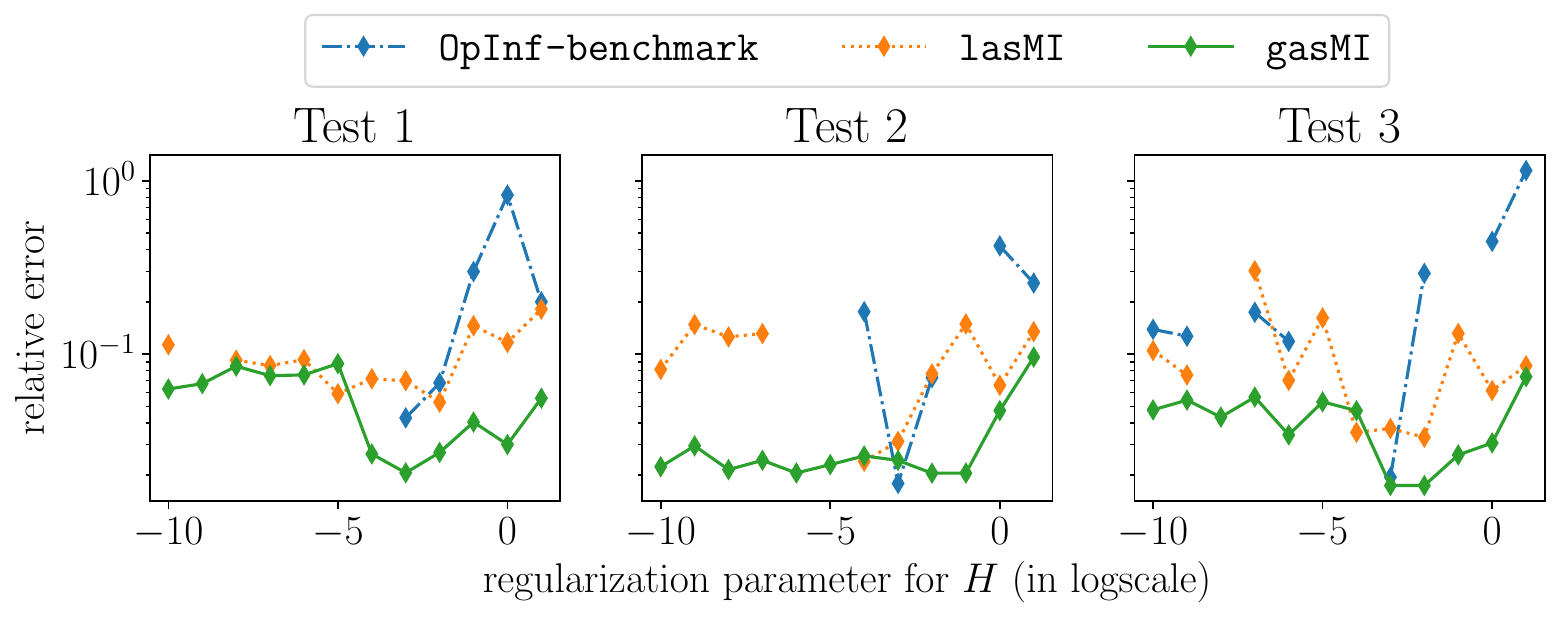}
	\caption{Burgers' equation  with Dirichlet boundary conditions: A performance over all the test data of the inferred models. The missing dots indicate instability of the model.}
	\label{fig:Burgers_Dirichilet_time_domain}
\end{figure} 

\begin{figure}[!tb]
	\centering
	\begin{subfigure}[t]{0.8\textwidth}
		\includegraphics[width = 0.95\textwidth]{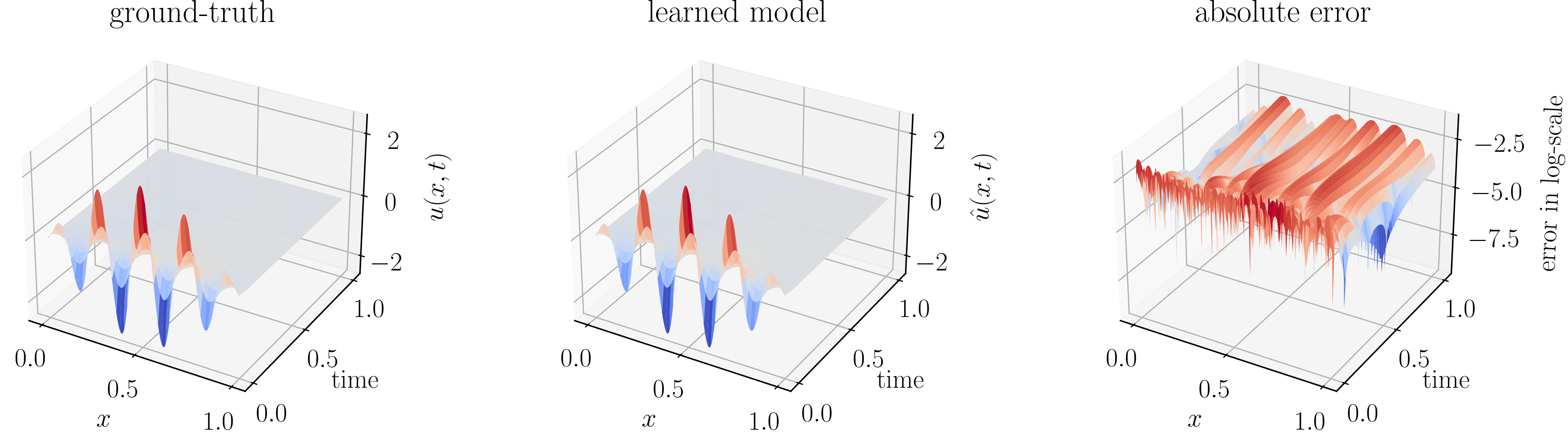}
		\caption{Using \opinfbenchmark.}
	\end{subfigure}
	\begin{subfigure}[t]{0.8\textwidth}
		\includegraphics[width = 0.95\textwidth]{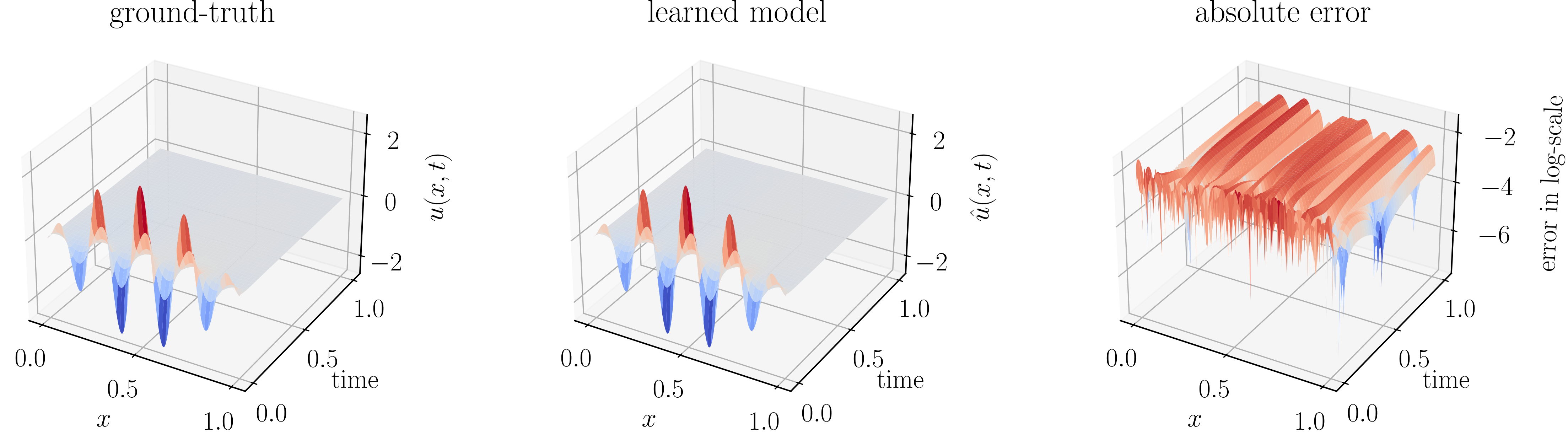}
		\caption{Using \lasmi.}
	\end{subfigure}
	\begin{subfigure}[t]{0.8\textwidth}
		\includegraphics[width = 0.95\textwidth]{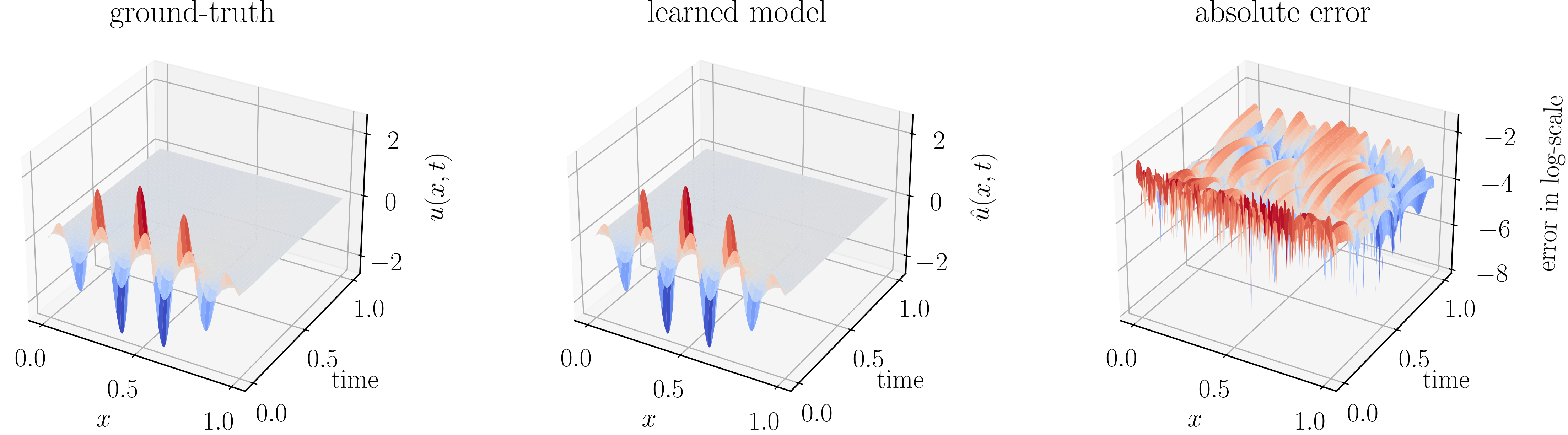}
		\caption{Using \gasmi.}
	\end{subfigure}
	\caption{Burgers' equation with Dirichlet boundary conditions: A comparison of the time-domain simulations of the inferred models on an initial test condition.}
	\label{fig:burgers_time_domain_onetraj}
\end{figure}
In our second study, we fix $\lambda_\bH$ to $10^{-3}$ and vary the order of the learned models from $16$ to $30$ in increments of $2$. We then infer the corresponding operators using \opinfbenchmark, \lasmi, and \gasmi. We determine the quality of the models on the test cases by taking the mean of the relative $L_2$-measure given in \cref{eq:l2_measure_test} over the test cases. We show the results in \Cref{fig:burgers_time_domain}. It demonstrates that \opinfbenchmark\ performs poorly for high-order models as compared to \las~and \gas. Furthermore, we notice a superior performance of \gas~for all orders.

Moreover, we highlight that \lasmi~learns models, which are \las~ which can be checked via the eigenvalues of the matrix $\bA$ in the system \eqref{eq:quad_model}. None of the stability properties can be guaranteed by \opinfbenchmark, and to show this, we closely look at the eigenvalues of the learned linear operator $\bA$ using all three methods. For this, we plot the eigenvalues of the linear operator $\bA$ for order $r=20$ by projecting them on the unit circle in \Cref{fig:burgers_eigenvalues}. We can notice that \opinfbenchmark~is  locally unstable, whereas \lasmi~and \gasmi~yield locally and globally stable models, respectively. 

\begin{figure}[tb]
	\centering
	\begin{subfigure}[t]{0.49\textwidth}
		\centering
			\includegraphics[height = 6cm]{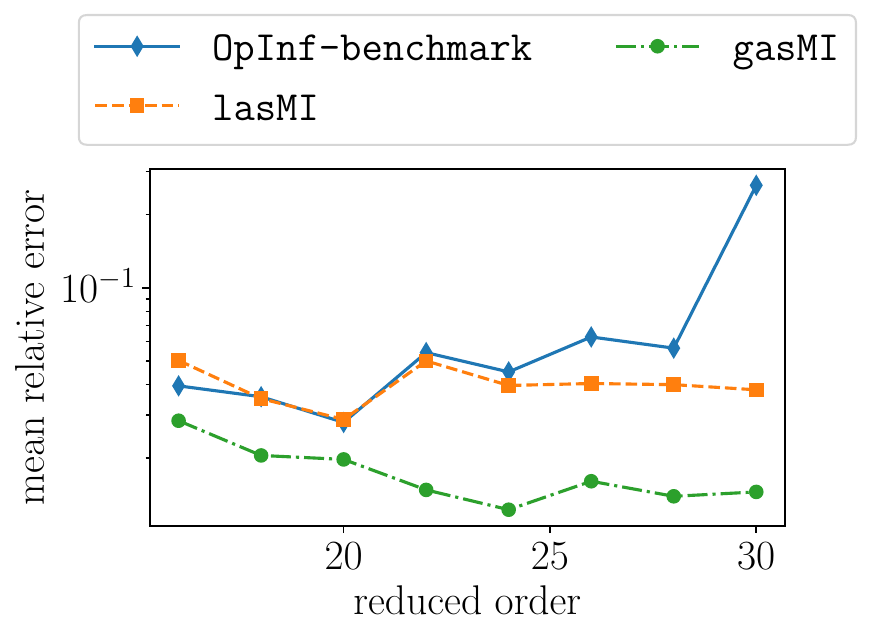}
		\caption{A mean performance over all the test data of the inferred models.}
		\label{fig:burgers_time_domain}
	\end{subfigure}
	\begin{subfigure}[t]{0.49\textwidth}
		\centering
		\includegraphics[height=6cm]{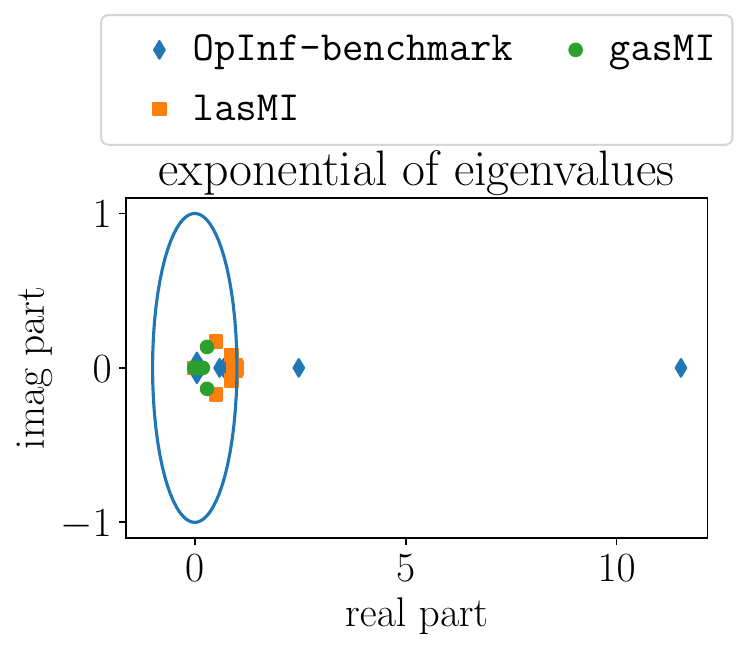}
		\caption{The eigenvalues of the inferred  linear operators but projected on the unit circle.}
		\label{fig:burgers_eigenvalues}
	\end{subfigure}
\caption{Burgers' equation with Dirichlet boundary conditions: Comparisons of the qualities of the inferred models.}
\end{figure}

\subsubsection{Chafee-Infante example}
In our second example, we consider the one-dimensional Chafee-Infante equation, which is  governed by the following equations:
\begin{equation}
\begin{aligned}
v_t + v^3  &=  v + v_{\zeta\zeta} & & \text{in}~(0,1)\times (0,T), \\
v_\zeta(0,\cdot)  & = 0,~~\text{and}~~v_\zeta(1,\cdot) &= 0, &&\\
v(\zeta,0) &= v_0(\zeta) & & \text{in}~(0,1),
\end{aligned}
\end{equation}
where, similar to the Burgers' equation, $v_t$ and $v_{\zeta\zeta}$ denote the derivative of $v$ with respect to\ the time $t$, and second derivative of $v$ with respect to the space $\zeta$. It is noteworthy that the governing equation exhibits a cubic nonlinearity. However, these equations can be written as a quadratic system by defining an auxiliary state variable as $\alpha v^2$, where $\alpha$ is a non-zero scalar, using lifting principles \cite{morGu09,morGu11,morBenB15,morBenG17,morBenGG18,morKraW19}.

Towards creating training and testing datasets, we first discretize the governing PDE using $1~000$ equidistant points for a finite-difference scheme. Then, we consider different initial conditions to collect data, as we did for the Burgers' equation in the previous example. We parameterize the initial condition as follows:
\begin{equation}
v(\zeta,0) = 0.1 + f\cdot(\sin(4\pi \zeta))^2,
\end{equation}
for a given $f$. We consider $13$ initial conditions by taking $13$ equidistant points for $f$ in the interval $[1,3]$. Assuming the considered values for $f$ are sorted in increasing order, we take the $4$th, $8$th, and $11$th indices for testing and the remaining $10$ values for training. For each initial condition, we take $500$ points in the time interval $[0,8]$.

To learn models, we use a lifting transformation as done in \cite{morGu09, morBenB15, QKPW2020_lift_and_learn}. To this aim,  we augment the state using $w := \alpha v^2$ so that a quadratic model can jointly describe the dynamics of both $v$ and $w$. Before that, we also shift $v$ by $1$ to have zero as an equilibrium point.
Moreover, we set $\alpha$ to $0.5$ so that the singular values of $v$ and $w$ are of the same order (at least the dominant ones); otherwise, we observed numerical difficulties in inferring the operators correctly. We first plot the decay of the singular values for $v$ and $w$ in \Cref{fig:chafee_svd}. 
\begin{figure}[tb]
	\centering
	\includegraphics[width = 0.5\textwidth]{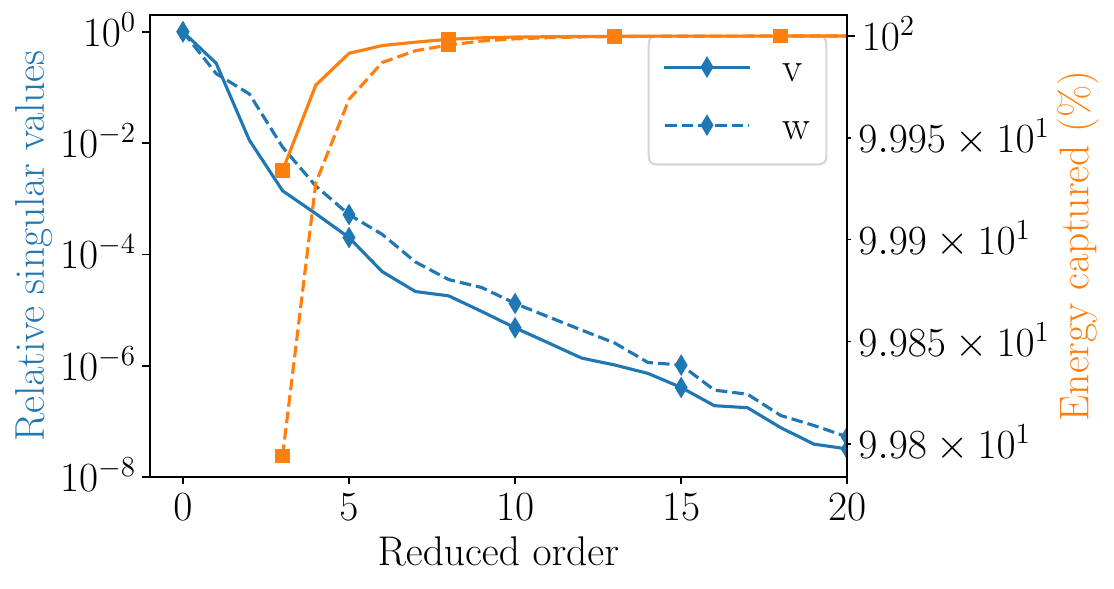}
	\caption{Chafee Infante equation: Decay of the singular values for the $v$ and $w := \tfrac{1}{2}v^2$ using the training data.}
	\label{fig:chafee_svd}
\end{figure} 
The figure indicates a rapid decay of the singular values, confirming a low-dimensional representation of the data should be possible. After that, we determine the dominant subspaces for $v$ and $w$ based on how much energy is aimed to capture. We stress that we determine dominant subspaces for $v$ and $w$ separately---denoted by $\bV_{\text{p}}$ and $\bW_{\text{p}}$---so that information of these two quantities are not mixed. In our experiments, this yields a more robust inference. We then compute the projection matrix $\bV := \diag{\bV_{\text{p}}, \bW_{\text{p}}}$. Using the matrix $\bV$, we determine a low-dimensional coordinate system for which we aim to learn reduced operators. By varying the amount of the energy captured by $\bV$, we obtain different dimensions of the reduced coordinate system, thus the order of learned models. We learn reduced-order operators using all three considered methods. The performance of these models is determined using the left-out test data. For reduced-order $r=7$, we plot the time-domain simulation for one of the test initial conditions in \Cref{fig:Chafee_time_domain_onetraj}. %
\begin{figure}[!tb]
	\centering
	\begin{subfigure}{0.9\textwidth}
		\includegraphics[width = 0.95\textwidth]{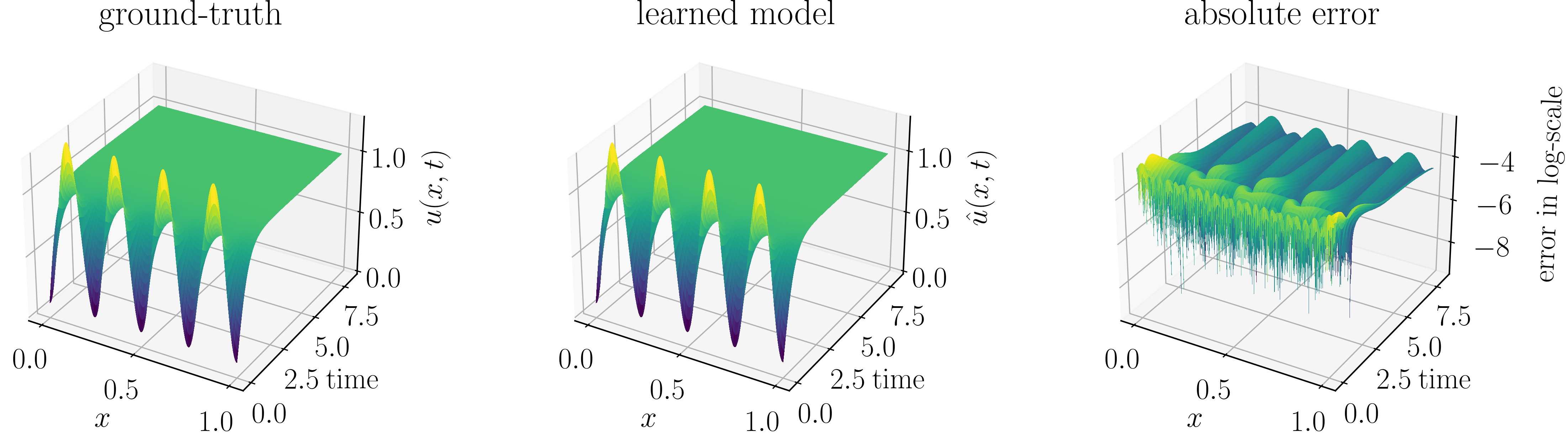}
		\caption{Using \opinfbenchmark.}
	\end{subfigure}
	\begin{subfigure}{0.9\textwidth}
		\includegraphics[width = 0.95\textwidth]{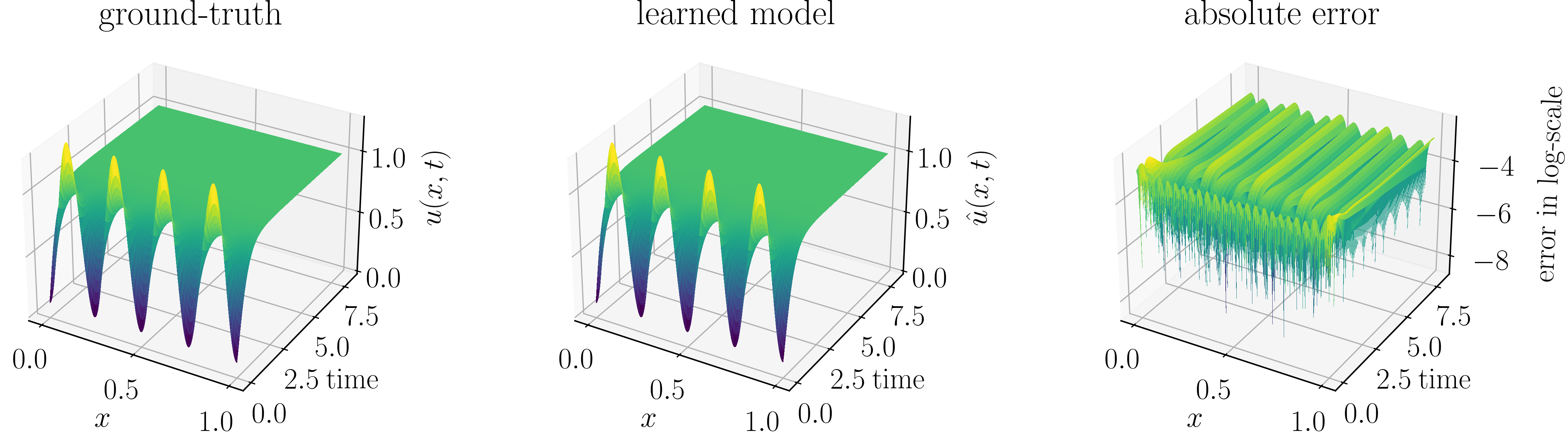}
		\caption{Using \lasmi.}
	\end{subfigure}
	\begin{subfigure}{0.9\textwidth}
		\includegraphics[width = 0.95\textwidth]{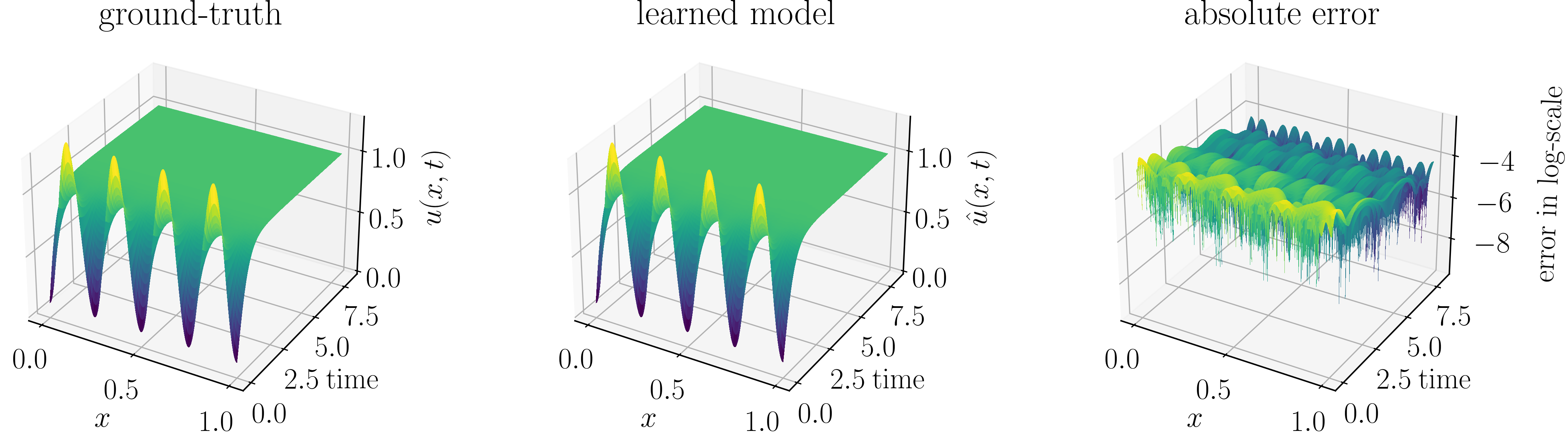}
		\caption{Using \gasmi.}
	\end{subfigure}
	\caption{Chafee-Infante equation: A comparison of the time-domain simulations of the learned models on an initial condition for testing.}
	\label{fig:Chafee_time_domain_onetraj}
\end{figure}

\begin{figure}[!tb]
	\centering
	\includegraphics[width = 0.5\textwidth]{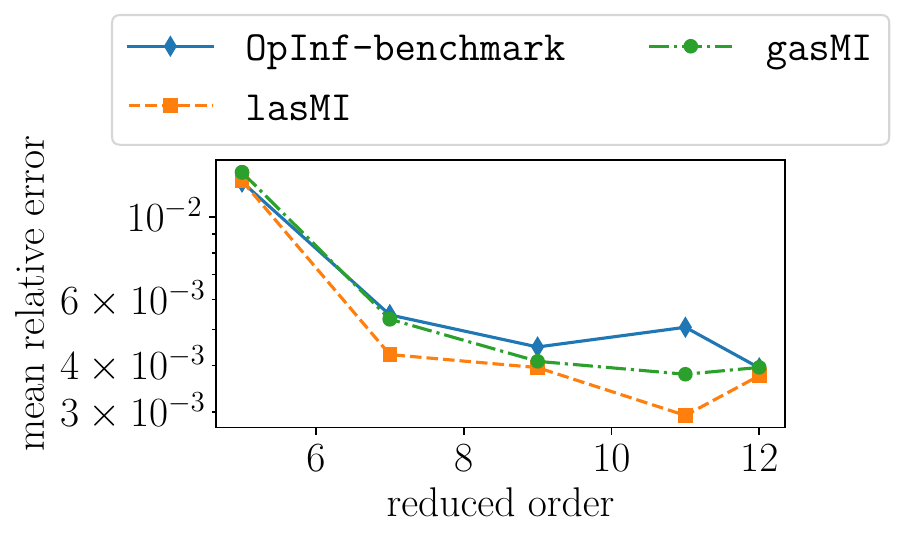}
	\caption{Chafee-Infante equation: A mean performance over all the test data using the inferred models. }
	\label{fig:Chafee_time_domain}
\end{figure} 

Moreover, the mean error over all the initial testing conditions---similar to the previous example---is plotted in \Cref{fig:Chafee_time_domain}. For this particular example, we observe that all the inferred models were locally stable, including for \opinfbenchmark, although it does not explicitly impose the stability. However, we notice a better performance of \lasmi~and \gasmi~for all three test cases as compared to \opinfbenchmark. Moreover, \gasmi---which guarantees global stability---also produces comparable models to \lasmi~despite \gasmi~posing more restrictions on the operators. 

\subsection{Discovery of governing equations via sparse regression}
Next, we test the proposed stability-guaranteed learning in the context of sparse regression to discover governing equations. 
\subsubsection{Chaotic Lorenz models}
The Lorenz model is one of the simplest models describing chaotic dynamics with roots in fluid dynamics \cite{lorenz1963deterministic}. Its dynamics are given by 
\begin{equation}
\begin{bmatrix}
\dot{\bx}(t) \\ \dot{\by}(t) \\ \dot{\bz}(t)
\end{bmatrix} = \begin{bmatrix}
-10\bx(t) + 10\by(t) \\ \bx(t)(28-\bz(t)) - \by(t) \\ \bx(t)\by(t) - \dfrac{8}{3}\bz(t)
\end{bmatrix}.
\end{equation}
We gather $5~000$ data points in the time interval $t = 0$ to $t=20$ with the initial condition $[-8.0, 7.0, 27.0]$. We corrupt the data by adding Gaussian noise of mean $0$ and standard deviation $0.1$. 
Like in \cite{goyal2022discovery}, we shift and re-scale the data as follows:
\begin{equation}
\tilde\bx(t) = \dfrac{\bx(t)}{8},\quad \tilde\by(t) = \dfrac{\by(t)}{8},\quad \text{and}\quad \tilde\bz(t) = \dfrac{\bz(t)-25}{8}. 
\end{equation}
With the shifting, we bring the mean and standard deviation of all the variables close to $0$ and $1$, respectively, without changing the interaction pattern, see \cite{goyal2022discovery}.

Note that the Lorenz system does not have any stable equilibrium point, but it is globally bounded and stable. Thus, it allows us to employ the trapping globally stable arguments discussed in \Cref{sec:atr}.  We employ \texttt{RK4-SINDy} \cite{goyal2022discovery}, inspired by the sparse regression for system identification (SINDy) approach \cite{brunton2016discovering} so that we can discover the underlying model without requiring derivative information. Instead of derivative information, \texttt{RK4-SINDy} also uses the integral form and approximates it using a fourth-order Runge-Kutta scheme. Note that there exist other methodologies, see, e.g., \cite{messenger2021weak}, which also avoid the need for the derivative information to discover governing equations in continuous time.

Next, we combine the discussion in Subsection~\ref{sec:atr} with sparse regression, where we seek to determine sparse matrices $\bA$ (or $\bJ$ and $\bR$), $\bH$, and $\bB$. To that end, we use an iterative thresholding scheme \cite{brunton2016discovering,goyal2022discovery}, in which small coefficients are hard-pruned. 
We learn models using both approaches and compare their performance by simulating for different initial conditions. We consider three test initial conditions as $[10,10,-10]$, $[100,-100,100]$, and $[-500,500,500]$. Note that these test initial conditions are even of a different magnitude than the training one. We plot the corresponding responses in \Cref{fig:Lorenz_comparison}, where we observe that \atrmi~yields models which are globally stable by construction and can faithfully capture the Lorenz dynamics. The \atrmi~~model results in dynamics correctly on the attractor. In contrast, \texttt{RK4-SINDy} does not guarantee global stability, and for very large initial conditions (for example, in 2-norm), it produces an unstable trajectory. 
\begin{figure}[!tb]
	\centering
	\begin{subfigure}{0.95\textwidth}
		\centering
		\includegraphics[width = 0.32\textwidth]{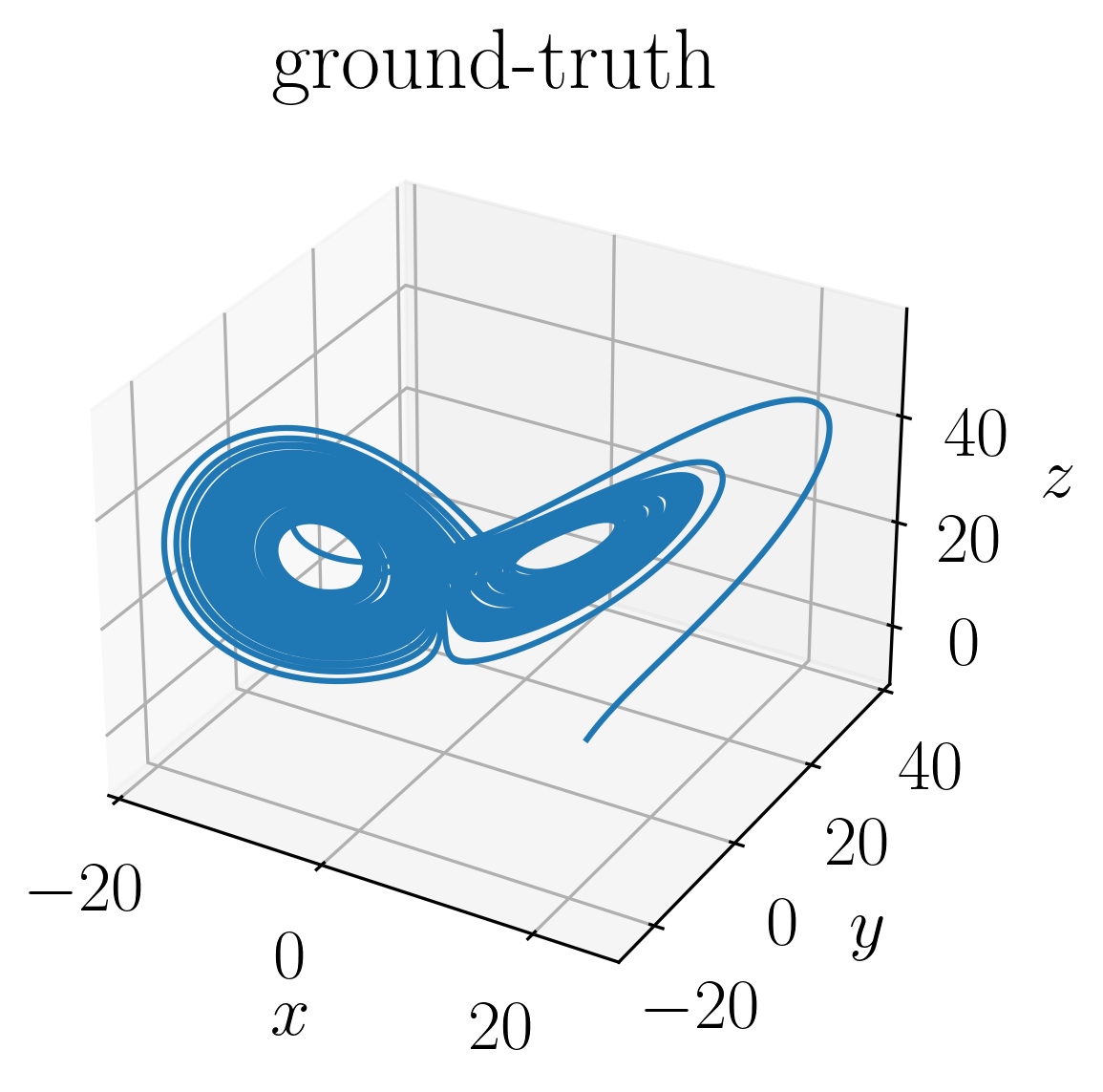}
		\includegraphics[width = 0.32\textwidth]{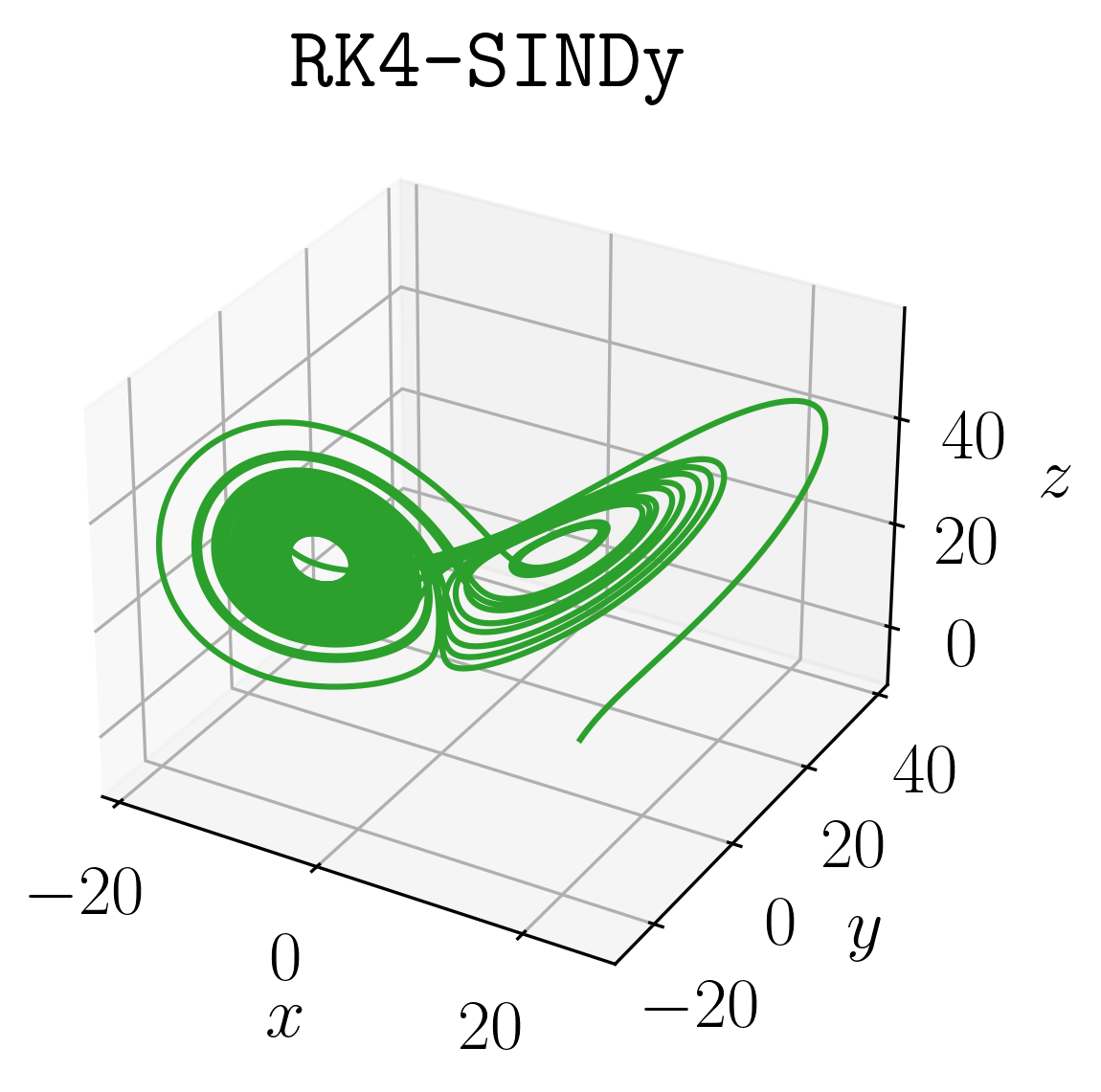}
		\includegraphics[width = 0.32\textwidth]{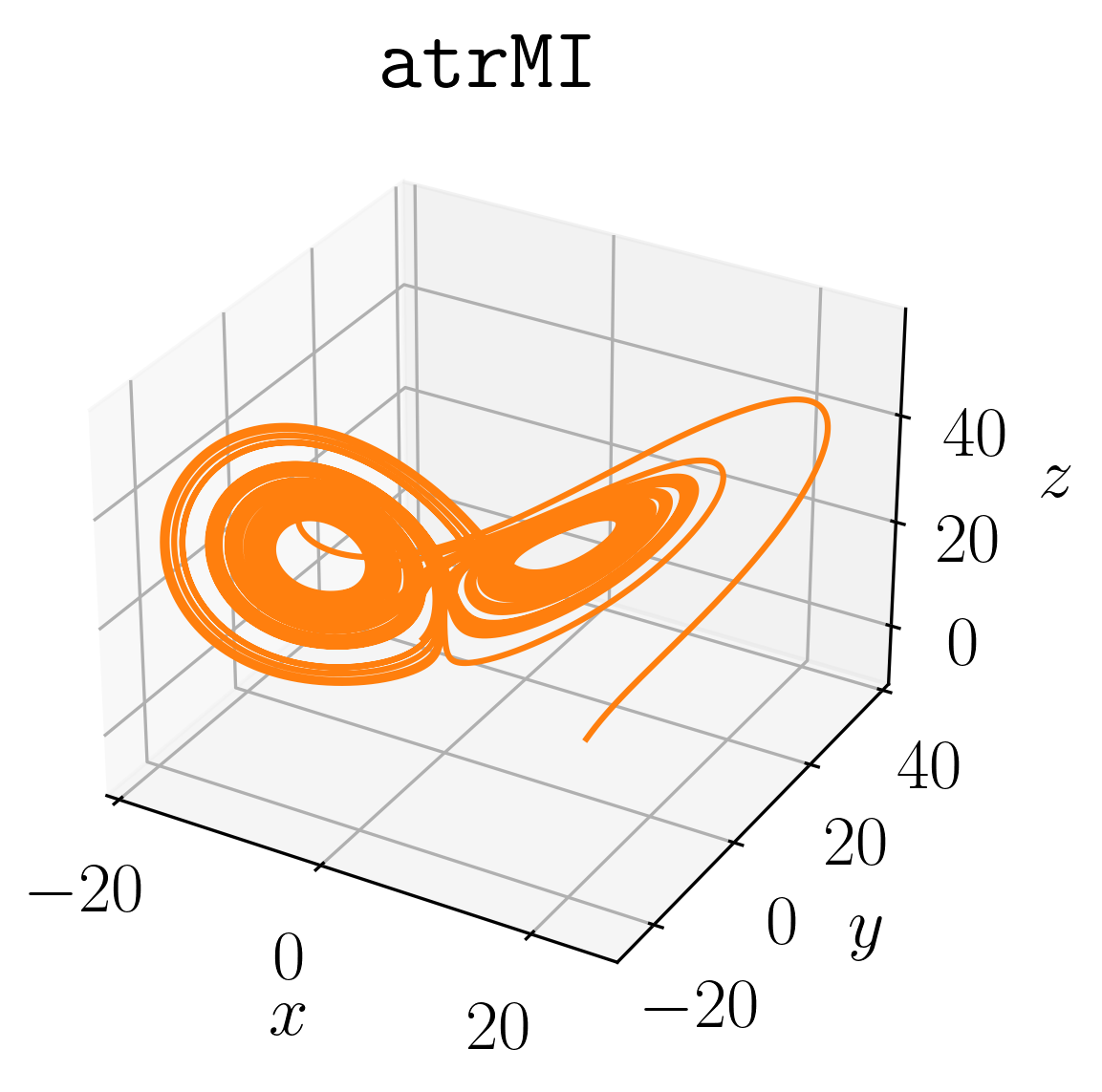}
		\caption{For initial condition $[10,10,-10]$.}
		\label{fig:lorenz_first}
	\end{subfigure}
	\begin{subfigure}{0.95\textwidth}
		\centering
	\includegraphics[width = 0.32\textwidth]{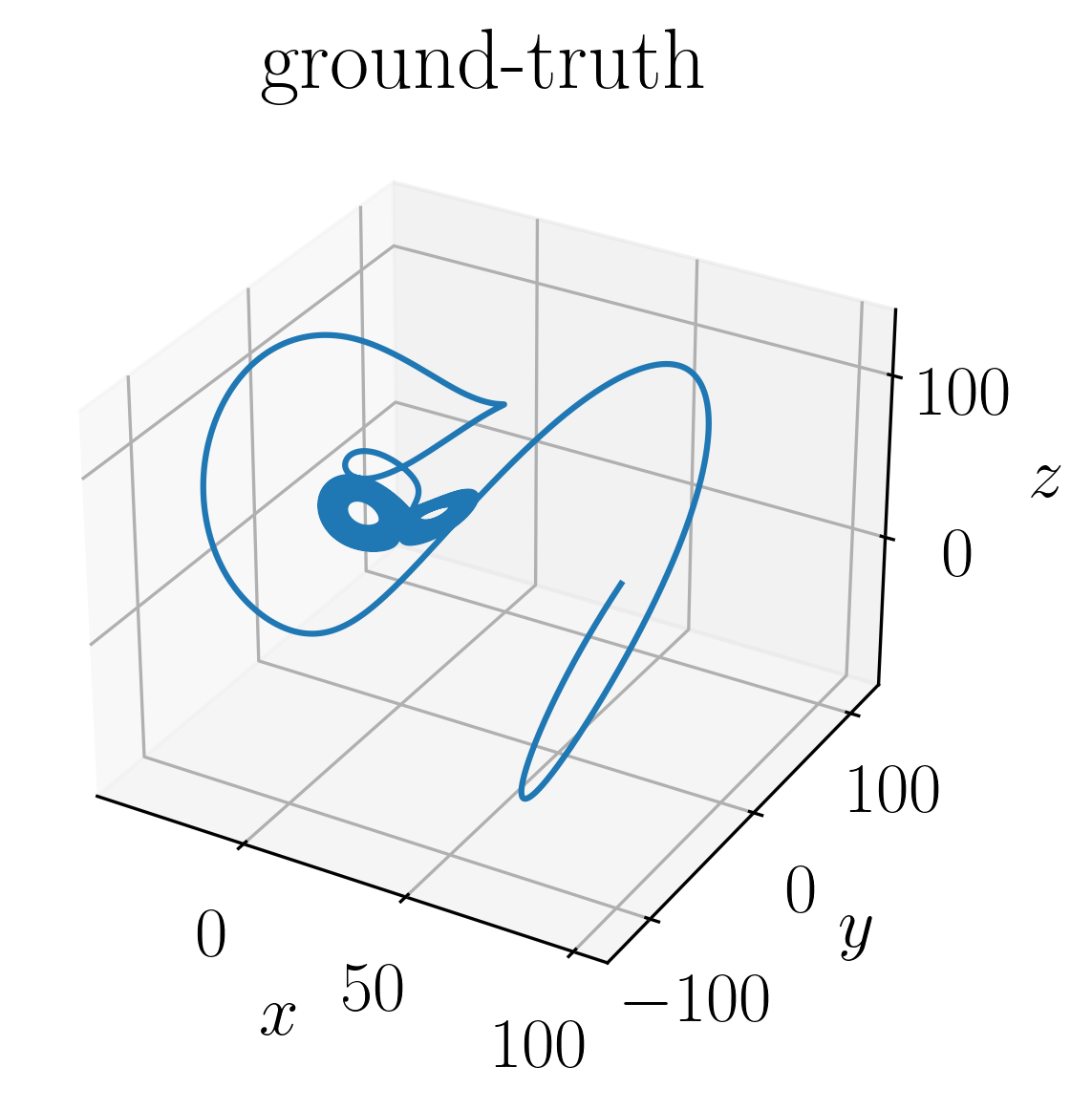}
	\includegraphics[width = 0.32\textwidth]{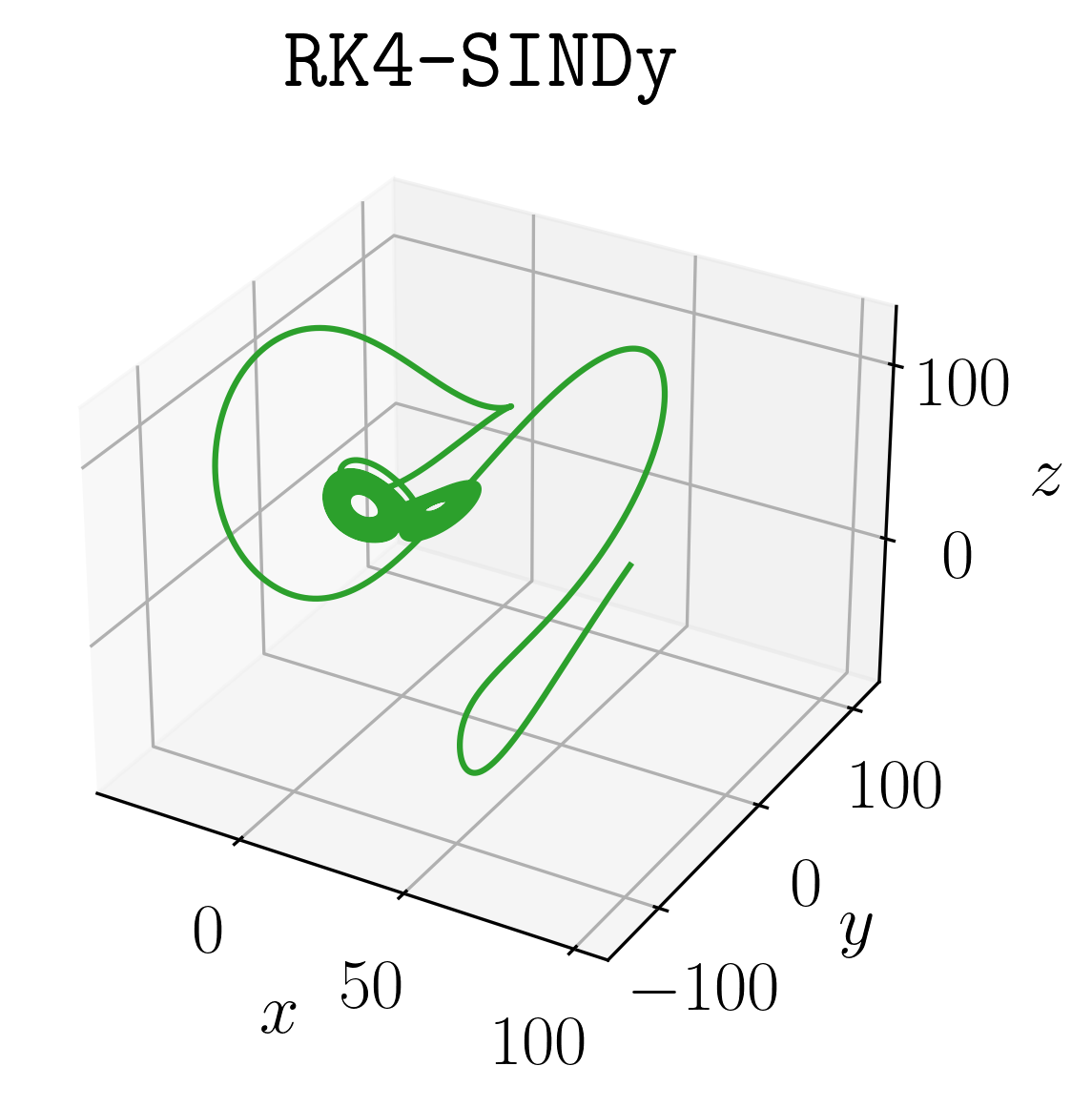}
	\includegraphics[width = 0.32\textwidth]{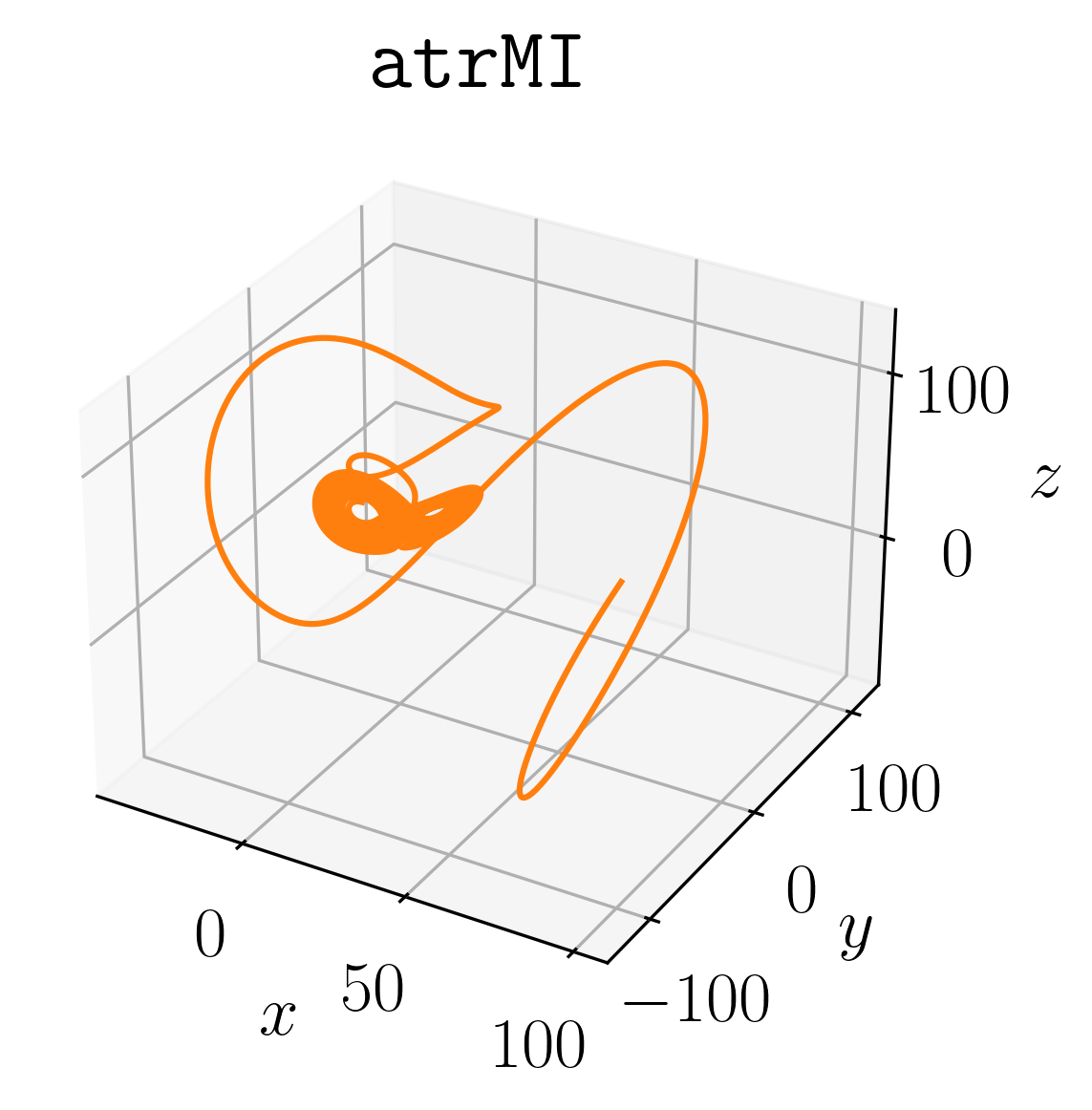}
	\caption{For initial condition $[100,-100,100]$.}
	\label{fig:lorenz_second}
\end{subfigure}
	\begin{subfigure}{0.95\textwidth}
		\centering
	\includegraphics[width = 0.32\textwidth]{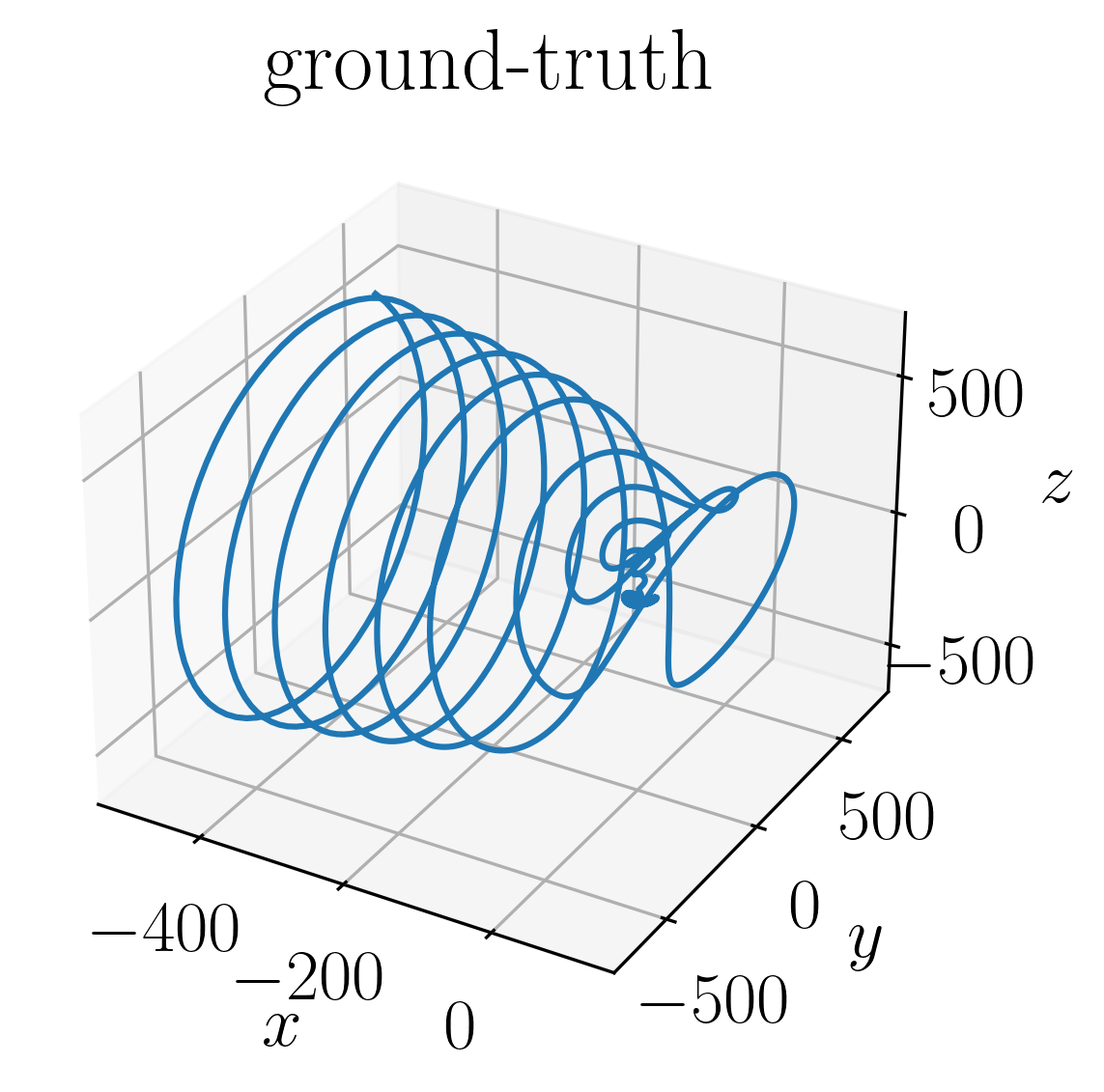}
	\includegraphics[width = 0.32\textwidth]{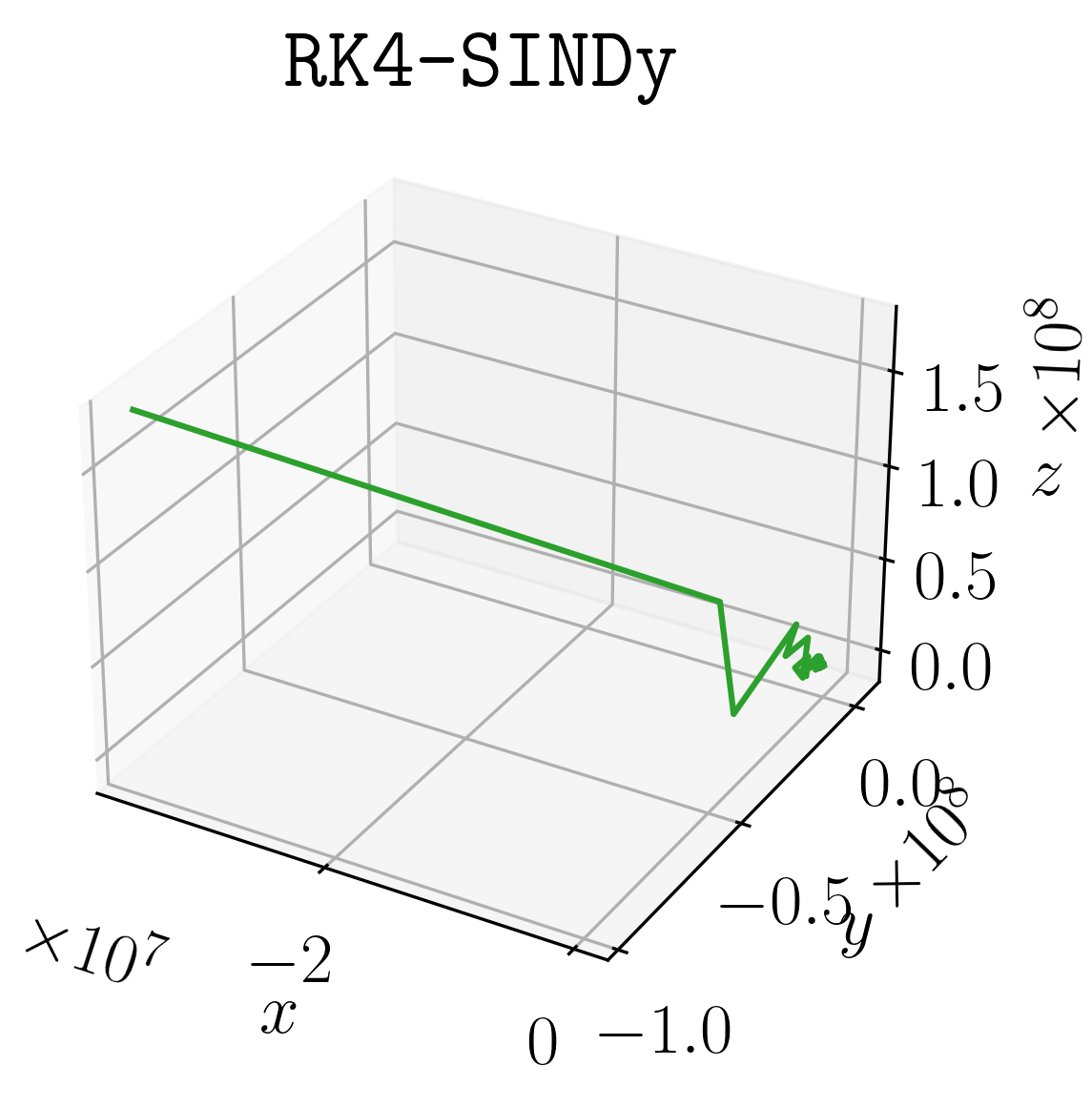}
	\includegraphics[width = 0.32\textwidth]{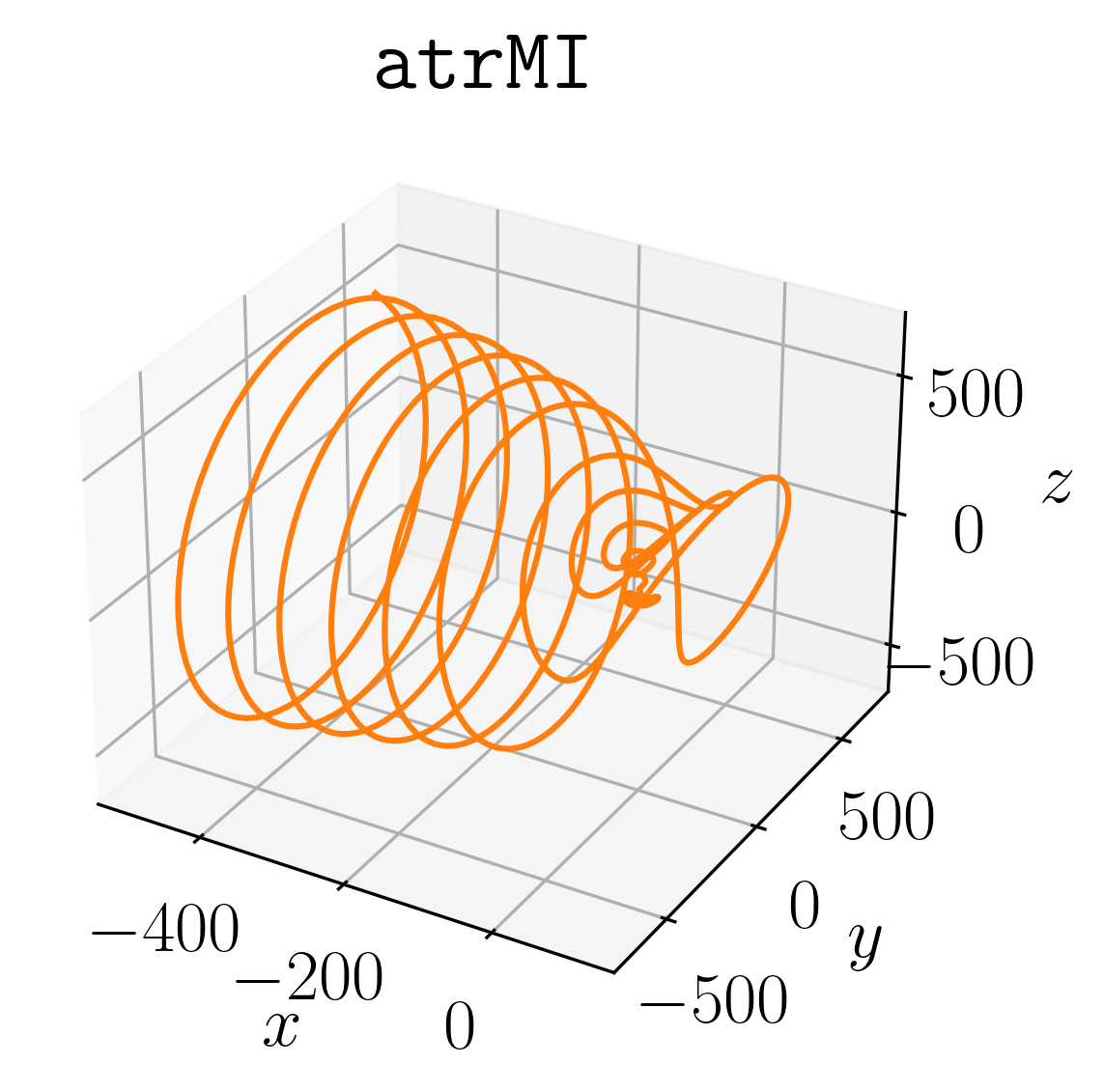}
	\caption{For initial condition $[-500,500,500]$.}
	\label{fig:lorenz_third}
\end{subfigure}
	\caption{Lorenz model: A comparison of the time-domain simulations of the learned models for test initial conditions.}
	\label{fig:Lorenz_comparison}
\end{figure}

\subsubsection{Magneto-hydrodynamic model}
In our next example, we aim to discover a magneto-hydrodynamic (MHD) model with quadratic non-linearities. The model exhibits energy-preserving properties. In this paper, we consider a simple two-dimensional incomprehensible MHD model from \cite{carbone1992relaxation,kaptanoglu2021promoting}, which is given by 
\begin{equation}\label{eq:MHD_model}
\begin{bmatrix}
\dot{v}_1(t) \\\dot{v}_2(t) \\\dot{v}_3(t) \\ \dot{b}_1(t) \\\dot{b}_2(t) \\\dot{b}_3(t)
\end{bmatrix} = \begin{bmatrix}
-2\nu & 0& 0 & 0& 0& 0 \\
0 & -5\nu& 0 & 0& 0& 0 \\
0 & 0& -9\nu & 0& 0& 0 \\
0 & 0& 0 & -2\mu& 0& 0 \\
0 & 0& 0 & 0& -5\mu& 0 \\
0 & 0& 0 & 0& 0& -9\mu
\end{bmatrix} \begin{bmatrix}
{v}_1(t) \\{v}_2(t) \\{v}_3(t) \\ {b}_1(t) \\{b}_2(t) \\{b}_3(t)
\end{bmatrix} + 
\begin{bmatrix}
4\left(v_2v_3 - b_2b_3\right) \\
-7\left(v_1v_3 - b_1b_3\right) \\
3\left(v_1v_2 - b_1b_2\right) \\
2\left(b_3v_2 - v_3b_2\right) \\
5\left(b_1v_3 - v_1b_3\right) \\
9\left(b_2v_1 - b_1v_2\right) 
\end{bmatrix},
\end{equation}
where $\nu \geq 0$ and $\mu\geq 0$ are the viscosity and resistivity, respectively. Typically, the system is stable in the absence of any external forces and dissipates to zero. Thus, the model is globally asymptotically stable. However, following \cite{kaptanoglu2021promoting}, we consider the in-viscid case by setting $\nu =\mu = 0$. As a result, the system is Hamiltonian and energy-preserving. Before we proceed further, we highlight that when the quadratic term of the model \eqref{eq:MHD_model} is noted down in the form of $\bH(\bx\otimes \bx)$, then it can have the form given in \eqref{eq:str_H} with $\bQ = \bI$. It fits into our hypothesis on the structure of $\bH$ in \eqref{eq:quad_model}.

Although \Cref{sec:Exp} discusses trapping regions in a strict sense, it also covers energy-preserving systems, where the energy $\tfrac{1}{2}(\bx-\bm)^\top \bQ(\bx-\bm)$ is preserved when $\bA_s = 0$. This implies that $\bR = 0$ in \Cref{sec:atr}; thus, we assume it to be zero, hence can be removed from the optimization problem \eqref{eq:stable_learning_atr1}. With this setting, we aim to infer an energy-preserving model with sparse regression and present a comparison with \texttt{RK4-SINDy} in \Cref{fig:MHD_comparison} for a testing condition, which is different from the initial condition used to generate the training data. 
We notice that the proposed methodology with parameterization (\atrmi) is able to infer a model that preserves the system's energy. In contrast, without any stability or energy-preserving enforcement, the inferred model using \texttt{RK4-SINDy} fails to capture the underlying Hamiltonian dynamics accurately in the test phase, although it has a good fit for the training data (see \Cref{fig:MHD_comparison} (b), upper triangle). 

\begin{figure}[!tb]
	\centering
	\begin{subfigure}{0.49\textwidth}
		\includegraphics[width = 0.95\textwidth]{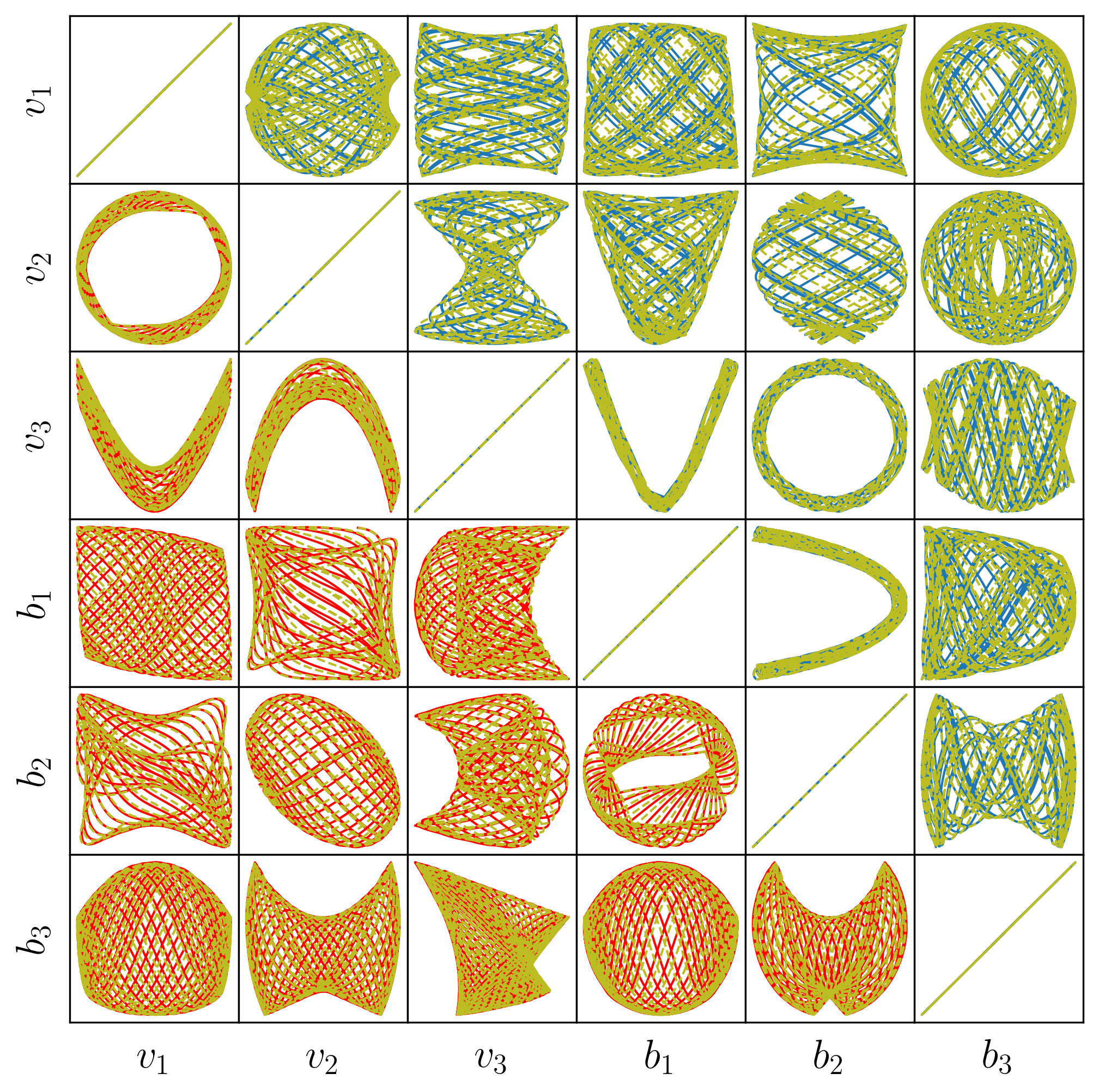}
		\caption{Using \texttt{atrMI}.}
		\label{fig:MHD_atrmi}

	\end{subfigure}
	\begin{subfigure}{0.49\textwidth}
		\includegraphics[width = 0.95\textwidth]{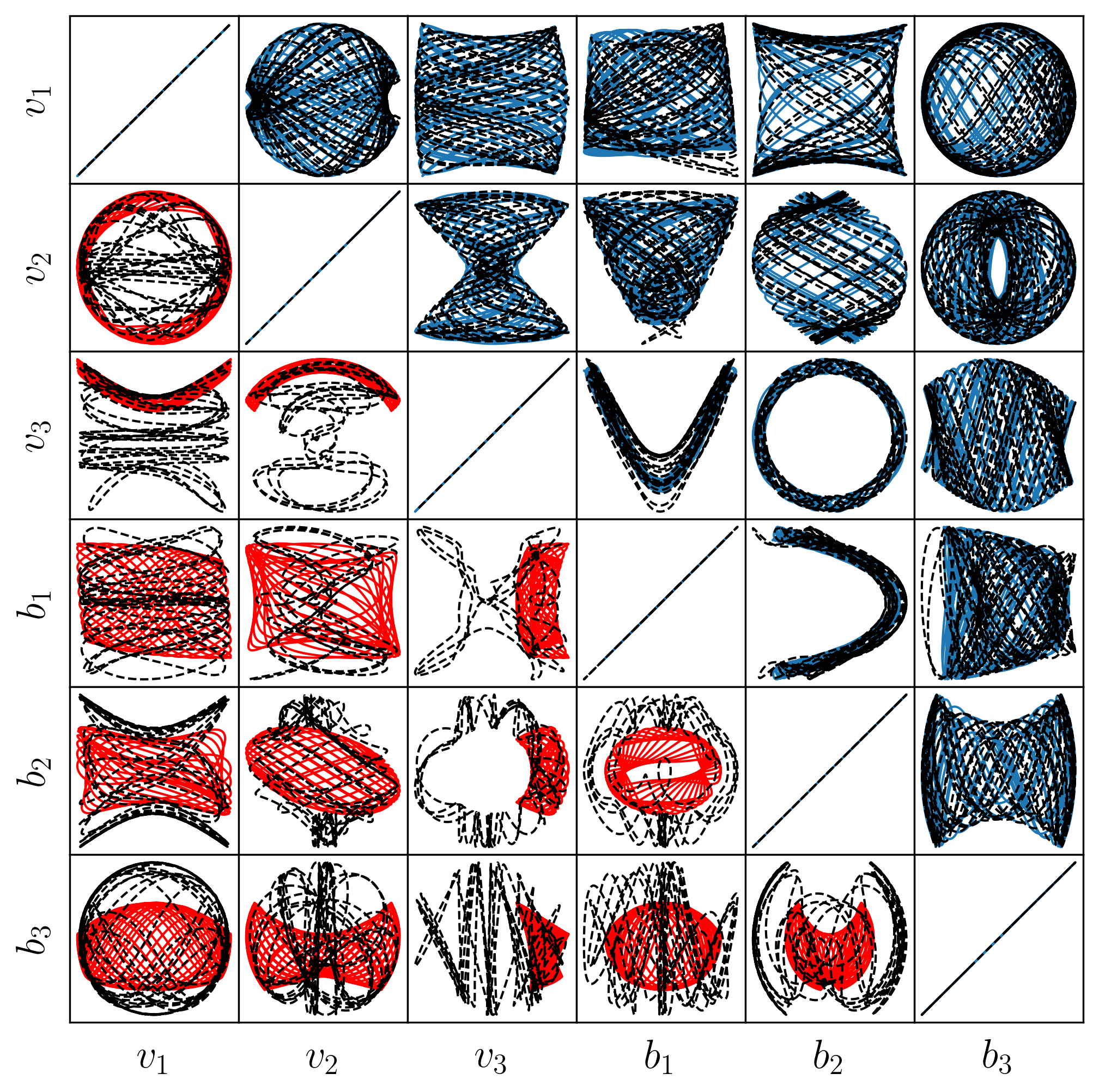}
		\caption{Using \texttt{RK4-SINDy}.}
		\label{fig:MHD_RK4sindy}
	\end{subfigure}
	\caption{MHD model: A comparison of the time-domain simulations of the learned models for test initial conditions. In both plots, the training data (blue, upper triangle), and testing data  (red, lower triangle) are shown, which are indeed very different. Figures (a) and (b) demonstrate the performance of the discovered model using \atrmi~(yellow) and \texttt{RK4-SINDy} (black) for both training and testing data.  }
	\label{fig:MHD_comparison}
\end{figure}

\subsection{Burgers equation with Neumann boundary conditions}
In our last example, we consider once again the one-dimensional Burgers' equation as presented in \eqref{eq:burgers}. However, we assume the following Neumann boundary conditions
\begin{equation}
\begin{aligned}
v_\zeta(0,\cdot)  & = 0, ~~\text{and~~}
v_\zeta(1,\cdot) = 0. &&\\
\end{aligned}
\end{equation}
In contrast to  the Dirichlet boundary condition from \Cref{sec:burgers_diri}, the nonlinear operator here is not energy-preserving, meaning it does not satisfy \eqref{eq:EnergyPreservQuadTerm}, and the system is not \gas. We consider the same setup for discretization as in \cite{morBenB15}. We consider $1~000$ equidistant grid points in the spatial domain, and $501$ equidistant data points are collected in the time-interval $[0,1]$. Furthermore, as in \cite{morGoyPB23}, we collect data using $17$ differential initial conditions, i.e., 
$$v_0(\zeta) = \alpha \cdot \cos(2\pi \zeta)^2, \qquad \alpha = \{0.8, 1.0,\ldots, 3.8,4.0\}.$$
Taking data corresponding to three initial conditions ($\alpha = \{1.6, 2.4,3.2\}$) out for testing, we use the rest of them for learning operators for quadratic systems of the form \eqref{eq:quad_model}.

Similar to the Burgers' example with the Dirichlet boundary conditions in~\Cref{sec:burgers_diri}, we plot the decay of singular values of the training data in \Cref{fig:burgers_svd_neumann}. We notice a rapid decay, thus a possibility of constructing low-order models yet capturing the underlying dynamics accurately. Next, we construct reduced models of order $20$ using all three considered methods. Note that we have used a regularizer based on the norm of the matrix $\bH$ as discussed in \Cref{sec:burgers_diri} with $\lambda_H = 10^{-3}$. The learned models are tested on the test initial conditions, and the results are reported in \Cref{fig:burgers_time_domain_neuman} by computing the mean $L_2$ errors over the test cases. Furthermore, the solutions on the full domain for one of the test cases are shown in \Cref{fig:burgers_time_domain_onetraj_neumann}. As expected, these figures demonstrate that \gasmi~is not suitable for this example, because its data comes from a dynamical system that is not \gas. 
On the other hand, \opinfbenchmark~and \las~yield comparable models, but \las~gives locally stable models.

\begin{figure}[!tb]
	\centering
	\includegraphics[width = 0.5\textwidth]{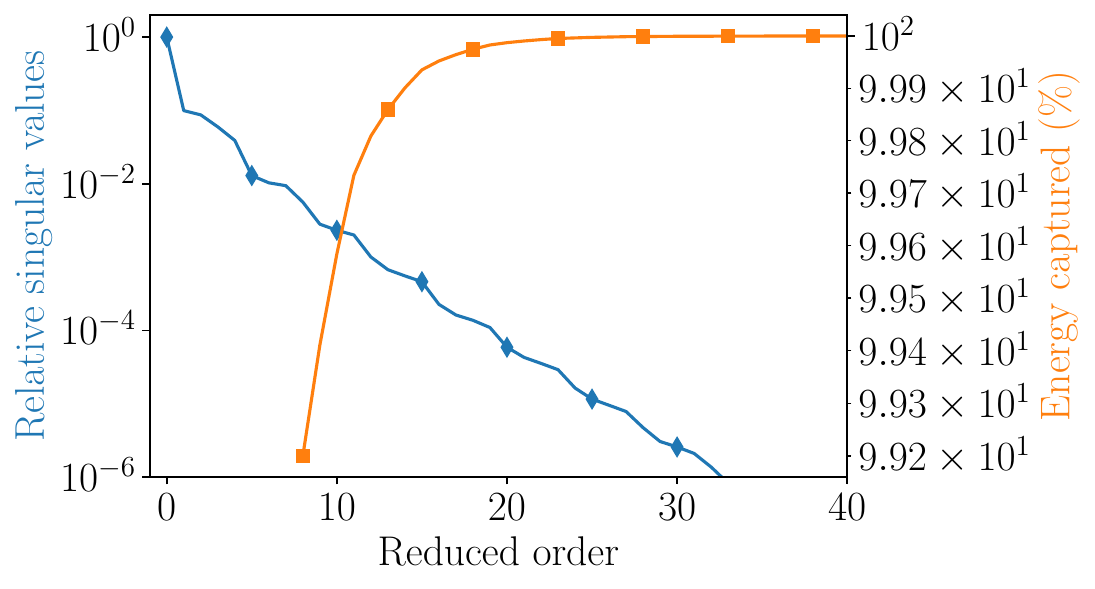}
	\caption{Burgers' equation with Neumann conditions: Decay of the singular values obtained using training data. The orange graph indicates how much energy is captured by how many dominant modes. }
	\label{fig:burgers_svd_neumann}
\end{figure} 

\begin{figure}[!tb]
	\centering
	\includegraphics[width = 0.5\textwidth]{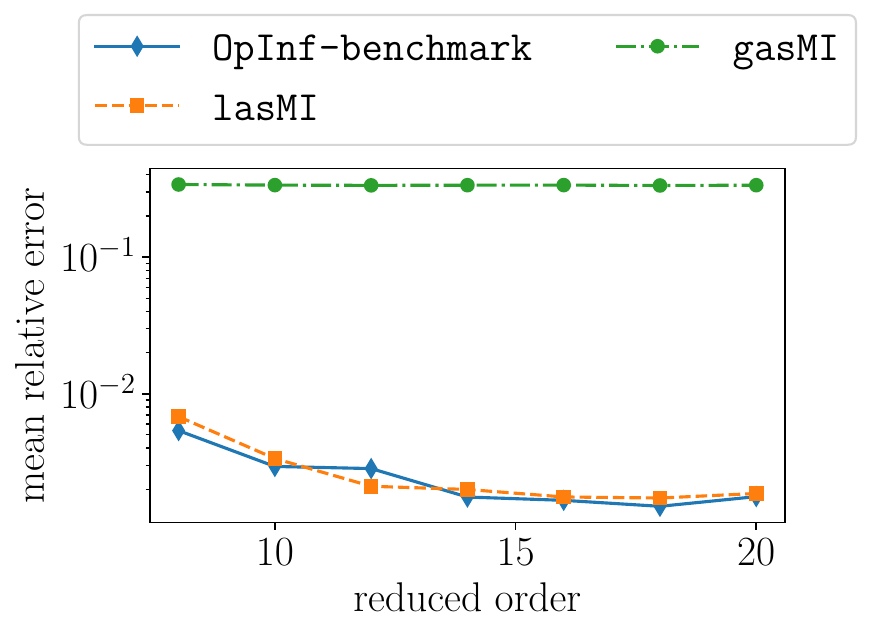}
	\caption{Burgers' equation  with Neumann conditions: A mean performance over all the test data of the inferred models.}
	\label{fig:burgers_time_domain_neuman}
\end{figure} 

\begin{figure}[!tb]
	\centering
	\begin{subfigure}[t]{0.9\textwidth}
		\includegraphics[width = 0.95\textwidth]{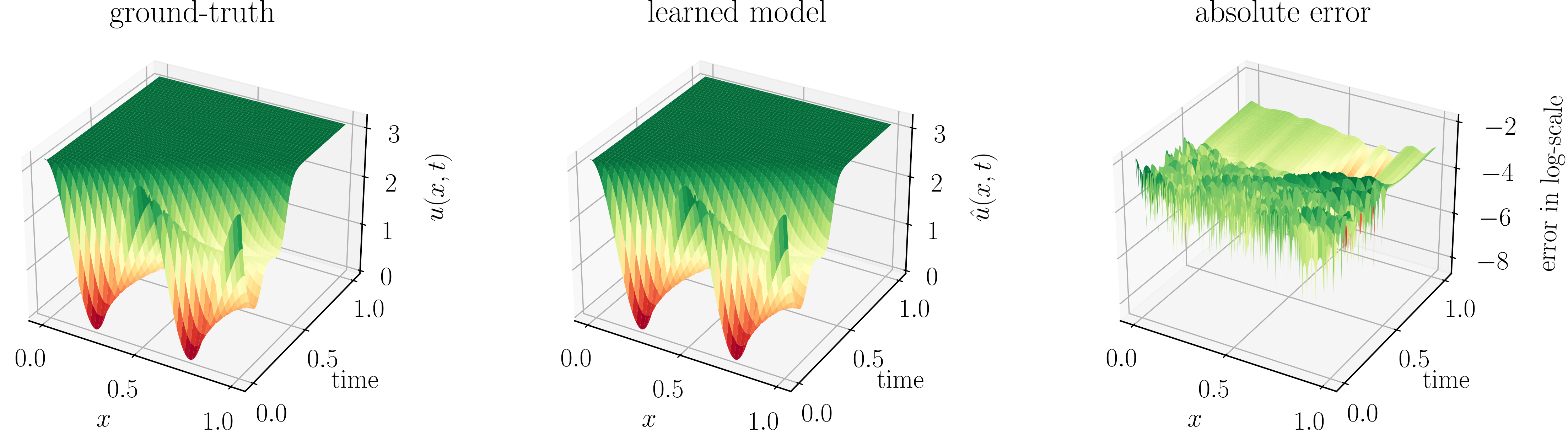}
		\caption{Using \opinfbenchmark.}
	\end{subfigure}
	\begin{subfigure}[t]{0.9\textwidth}
		\includegraphics[width = 0.95\textwidth]{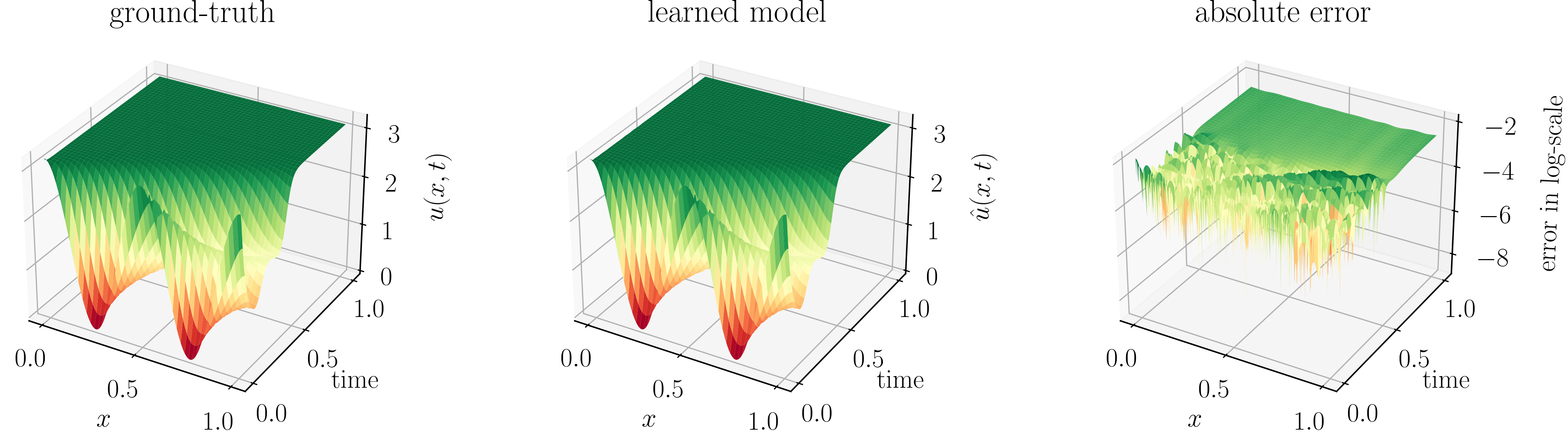}
		\caption{Using \lasmi.}
	\end{subfigure}
	\begin{subfigure}[t]{0.9\textwidth}
		\includegraphics[width = 0.95\textwidth]{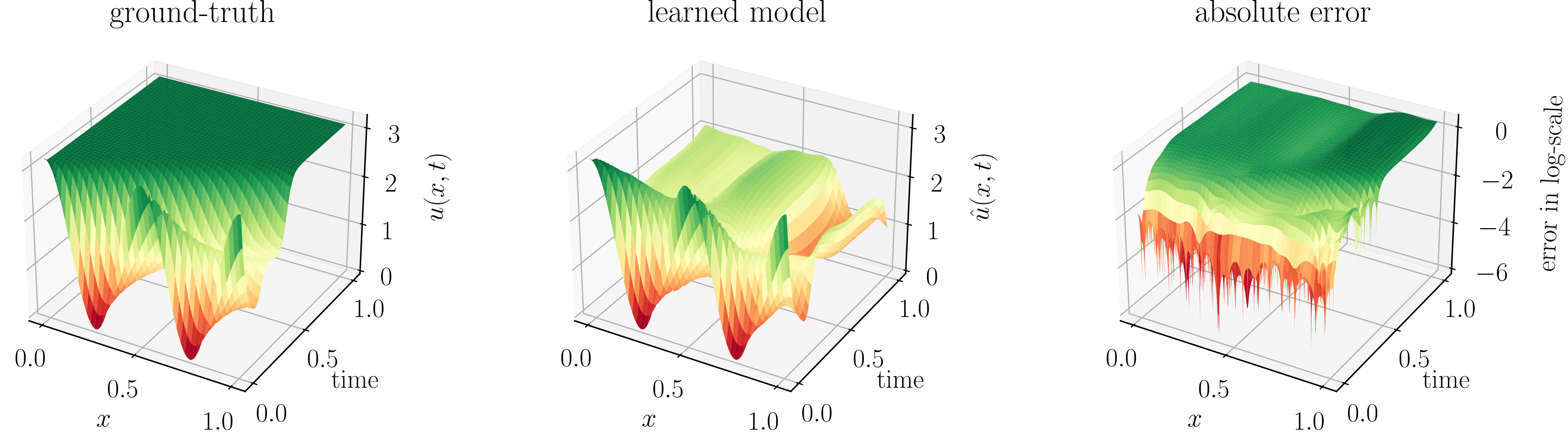}
		\caption{Using \gasmi.}
	\end{subfigure}
	\caption{Burgers' equation  with Neumann conditions: A comparison of the time-domain simulations of the inferred models on an initial test condition. It demonstrates a failure of \gasmi~for Neumann boundary conditions since the underlying model does not have an energy-preserving quadratic term.}
	\label{fig:burgers_time_domain_onetraj_neumann}
\end{figure}

\section{Conclusions}\label{sec:Conc}
This paper has discussed data-driven inference of quadratic systems that guarantee stability by design. We have begun by investigating the local and global asymptotic stability of quadratic systems using a generalized Lyapunov quadratic function. We have utilized the concept of energy-preserving non-linearity for global asymptotic stability. Based on these discussions, we have proposed suitable parameterizations, allowing us to learn quadratic systems that ensure local and global asymptotic stability. We have achieved these goals without using any matrix inequality-based constraint or based on eigenvalues of the matrix but by using dedicated parameterizations of operators. 
Additionally, we have addressed the problem of inferring quadratic systems with no stable equilibrium points. For this, we have discussed the concept of a global attracting trapping region, followed by proposing an appropriate parameterization.
Moreover, to avoid using derivative information while inferring the operators, we have blended a numerical integration scheme that is robust to noise and scarce data. 
Furthermore, since the involved optimization problems for inference do not have a closed-form solution, we have utilized a gradient-based approach, and the stability of the model is preserved at each iteration due to the stability-preserving dedicated parameterization. Several numerical examples demonstrate its effectiveness in inferring low-dimensional operators, discovering governing equations, and energy-preserving modeling.
In our future work, we plan to extend our frameworks for learning stable parametric operators and a detailed investigation of how the methodologies can be used in the case of highly noisy scenarios. Additionally, we aim to extend the proposed framework to a more general quadratic systems for which the non-linearity is not energy-preserving yet the model is bounded in some sense. In particular, we aim to investigate stable parametrization for fluid dynamical systems with Neumann boundary conditions. Overall, this paper provides a contribution to the field of quadratic system inference, offering a promising framework for developing stable quadratic models using data by appropriate parameterization. 
	
\bibliographystyle{elsarticle-num} 
\bibliography{mor,igorBiblio}

\appendix 
\section{Appendix}\label{appendix}
In the following, we show that equations \eqref{eq:energyPreserving_H_condition} and \eqref{eq:EnergyPreservQuadTerm} are equivalent.
\begin{theorem}\label{theorem:EquivalKron} The following statements are equivalent:
	\begin{enumerate}[(a)]
		\item  The $\bH$ matrix, in \eqref{eq:quad_model}, satisfies
		\begin{equation}\label{eq:Tensor_cond}
		\bH_{ijk} + \bH_{ikj} + \bH_{jik} + \bH_{jki} + \bH_{kij} + \bH_{kji} = 0, 
		\end{equation} 
		where  $ \{i,j,k\} \in \{1, \dots n\}$ and $\bH_{ijk} := e_i^{\top} \bH(e_j \otimes e_k)$.
		\item The $\bH$ matrix, in \eqref{eq:quad_model}, satisfies  
		\begin{equation}\label{eq:KronProd_cond}\bx^{\top}\bH(\bx\otimes \bx) = 0, \quad \text{for every $\bx \in \R^n$}. \end{equation}
	\end{enumerate}	
\end{theorem}
\begin{proof} $(a) \Rightarrow (b)$: For a given $\bx \in \R^n,$ let us write it as $\bx = \displaystyle\sum_{i=1}^n x_i  e_i$, where $x_i \in \R$ are the coordinates of $\bx$ in the basis $\{ e_1, \dots,  e_n\}$. Using this notation, the left-hand side of equation \eqref{eq:KronProd_cond} becomes
	\begin{align*}
	\bx^{\top}\bH(\bx\otimes \bx) &=  \left(\sum_{i=1}^n x_i   e_i^{\top}\right)\bH\left(\left(\sum_{j=1}^n x_j   e_j\right)\otimes \left(\sum_{k=1}^n x_k  e_k\right)\right) 	\\
	& =  \sum_{i=1}^n \sum_{j=1}^n  \sum_{k=1}^n  x_ix_jx_k   e_i^{\top}\bH\left(  e_j\otimes   e_k\right).
	\end{align*}
	Note that the product $x_ix_jx_k$ gives the same value if we permute the indexes in $(i,j,k)$. Hence, the above sum can be rearranged as follows:
	\begin{equation}\label{eq:permIdent}
	\begin{aligned}
	\sum_{i=1}^n \sum_{j=1}^n  \sum_{k=1}^n  x_ix_jx_k  e_i^{\top}\bH\left( e_j\otimes  e_k\right)   &= \\
	& \hspace{-3cm} =  \sum_{i=1}^n x_i^3 \left(\bH_{iii}\right)+ \dfrac{1}{2}  \sum_{i=1}^n \sum_{\substack{j=1 \\ j\neq i}}^n x_ix_j^2 \left(\bH_{ijj} + \bH_{jij} + \bH_{jji} + \bH_{ijj} + \bH_{jij} + \bH_{jji} \right) \\
	& \hspace{-2.5cm} +  \dfrac{1}{6}  \sum_{i=1}^n \sum_{\substack{j=1 \\ j\neq i}}^n  \sum_{\substack{k=1 \\ k\neq i\\ k\neq j}}^n  x_ix_jx_k \left(\bH_{ijk} + \bH_{ikj} + \bH_{jik} + \bH_{jki} + \bH_{kij} + \bH_{kji}\right),	
	\end{aligned}
	\end{equation}
	with $\bH_{ijk} := e_i^{\top} \bH(e_j \otimes e_k)$. Using \eqref{eq:Tensor_cond}, one can note the following:
	\begin{align*}
	\bH_{iii} & = 0\\
	\bH_{ijj} + \bH_{jij} + \bH_{jji} + \bH_{ijj} + \bH_{jij} + \bH_{jji} & =0 \\
	\bH_{ijk} + \bH_{ikj} + \bH_{jik} + \bH_{jki} + \bH_{kij} + \bH_{kji} & = 0,
	\end{align*}
	As a result, $\bx^{\top}\bH(\bx\otimes \bx)  =0$.
	
	$(b) \Rightarrow (a)$: Using \eqref{eq:KronProd_cond} and \eqref{eq:permIdent}, we have 
	\begin{equation*}
	\begin{aligned}
 &\sum_{i=1}^n x_i^3 \left(\bH_{iii}\right)+ \dfrac{1}{2}  \sum_{i=1}^n \sum_{\substack{j=1 \\ j\neq i}}^n x_ix_j^2 \left(\bH_{ijj} + \bH_{jij} + \bH_{jji} + \bH_{ijj} + \bH_{jij} + \bH_{jji} \right) \\
	& \hspace{0.5cm} +  \dfrac{1}{6}  \sum_{i=1}^n \sum_{\substack{j=1 \\ j\neq i}}^n  \sum_{\substack{k=1 \\ k\neq i\\ k\neq j}}^n  x_ix_jx_k \left(\bH_{ijk} + \bH_{ikj} + \bH_{jik} + \bH_{jki} + \bH_{kij} + \bH_{kji}\right) = 0.
	\end{aligned}
	\end{equation*}
	Since the above equation can be viewed as a multi-dimensional polynomial in $(x_1,\ldots, x_n)$ and it is zero, each coefficient must be zero. This leads to the condition $\eqref{eq:Tensor_cond}$.
\end{proof}

Next, we state that every energy-preserving matrix $\bH \in \R^{n \times n^2}$ can be represented using skew-symmetric matrices, without modifying the system dynamics. 

\begin{theorem}\label{thm:equvilence_between_H}
	Consider a energy-preserving matrix $\bH \in \R^{n\times n^2}$, satisfying \eqref{eq:energyPreserving_H_condition}. Then, there exists a matrix $\tilde \bH$ with structure as follows:
	\begin{equation}\label{eq:skew_symmetric_H}
	\tilde \bH = [\bH_1,\ldots, \bH_n],\qquad \bH_i = -\bH_i^\top  ~~\text{for}~~i = \{1,\ldots, n\},
	\end{equation}
	which is also energy-preserving and $\bH(\bx \otimes \bx) = \tilde \bH(\bx\otimes \bx)$ for every $\bx\in \R^n$.
\end{theorem}

\begin{proof}
The proof of the theorem has been initially attempted in \cite{Gki2023} in an internal report. In the following, we provide a new  proof for the completeness of the paper.
In this proof, we will show that a generic $\tilde \bH$ as in \eqref{eq:skew_symmetric_H} for which  $\bH(\bx \otimes \bx) = \tilde \bH(\bx\otimes \bx)$ is encoded by a linear system that has at least one solution. First note that $\tilde \bH$ satisfying \eqref{eq:skew_symmetric_H} is energy-preserving as it can be easily seen that $\bx^\top \tilde \bH (\bx\otimes \bx) = 0$. 
	Since both matrices $\bH$ and $\tilde \bH$ are energy-preserving, from  \eqref{eq:energyPreserving_H_condition}, we have
	\begin{subequations}\label{eq:energy_preserving_appendix}
		\begin{align}
		e_k^\top\bH(e_i\otimes e_j + e_j\otimes e_i) + e_j^\top\bH(e_i\otimes e_k + e_k\otimes e_i) + 	e_i^\top\bH(e_j\otimes e_k + e_k\otimes e_j) &= 0, \\		e_k^\top\tilde\bH(e_i\otimes e_j + e_j\otimes e_i) + e_j^\top\tilde\bH(e_i\otimes e_k + e_k\otimes e_i) + 	e_i^\top\tilde\bH(e_j\otimes e_k + e_k\otimes e_j) &= 0,
		\end{align}
	\end{subequations}
	for  $i,j,k \in \{1,\ldots,n\}$.
	Next, we notice that 
	\begin{equation}
	\bH(\bx \otimes \bx) = \tilde \bH(\bx\otimes \bx)
	\end{equation}
	if and only if
	\begin{equation}\label{eq:equvi_H}
	e_k^\top \bH(e_i \otimes e_j + e_j \otimes e_i) = e_k^\top \tilde\bH(e_i \otimes e_j + e_j \otimes e_i), ~~\text{for}~~{(i,j,k)} \in \{1,\ldots,n\}.
	\end{equation}

	Now, consider the case when $(i,j,k) \in \cM_1$, where $\cM_1 = \{(i, j, k) \mid i, j, k \in \{1, 2, \ldots, n\} \text{ and } i =j=k\}$. For this case, since $\tilde\bH$ and $\bH$ are energy-preserving, we have
	$$e_k^\top\bH(e_i \otimes e_j + e_j \otimes e_i) = e_k^\top\tilde\bH(e_i \otimes e_j + e_j \otimes e_i) = 0\qquad (i,j,k)\in \cM_1.$$
	
	Next, we focus on the cases when $(i,j,k) \in \cM_2$, where $\cM_2 = \{(i, j, k) \mid i, j, k \in \{1, 2, \ldots, n\} \text{ and } i \neq j \text{ or } i\neq k \text{ or } j\neq k\}$. First note that
	\begin{equation}
	e_k^\top \bH(e_i \otimes e_j + e_j \otimes e_i) = \left((e_i^\top \otimes e_j^\top + e_j^\top \otimes e_i^\top) \otimes e_k^\top\right)\bH_{\texttt{vect}},
	\end{equation}
	where $\bH_{\texttt{vect}}$ is the vectorization of the matrix $\bH$. Then, the conditions in \eqref{eq:equvi_H} for $(i,j,k) \in \cM_2$ can be written as follows:
	\begin{equation}\label{eq:equvi_H_Aform}
	\cA \bH_{\texttt{vect}} = \cA \tilde\bH_{\texttt{vect}},
	\end{equation}
	where $\cA$ is a linear operator defined so that its row vectors are constructed using $\left((e_i^\top \otimes e_j^\top + e_j^\top \otimes e_i^\top) \otimes e_k^\top\right)$ for all $(i,j,k) \in \cM_2$. Consequently, the matrix $\cA$ takes the following form:
	\begin{equation}
	\cA = \begin{bmatrix}
	\left((e_1^\top \otimes e_1^\top + e_1^\top \otimes e_1^\top) \otimes e_2^\top\right)\\
	\left((e_1^\top \otimes e_2^\top + e_2^\top \otimes e_1^\top) \otimes e_2^\top\right)\\
	\vdots \\
	\left((e_n^\top \otimes e_n^\top + e_n^\top \otimes e_n^\top) \otimes e_2^\top\right)\\
	\left((e_1^\top \otimes e_1^\top + e_1^\top \otimes e_1^\top) \otimes e_3^\top\right) \\
	\vdots \\ 
	\left((e_n^\top \otimes e_n^\top + e_n^\top \otimes e_n^\top) \otimes e_3^\top\right) \\
	\vdots \\
	\left((e_n^\top \otimes e_{n-1}^\top + e_{n-1}^\top \otimes e_n^\top) \otimes e_n^\top\right).
	\end{bmatrix}
	\end{equation}
	We note that the total number of rows in the matrix $\cA$  are $\tfrac{n^2(n+1)}{2} - n$. 
	
	Furthermore, we write the energy-preserving conditions \eqref{eq:energy_preserving_appendix} in a vectorized form for $\{i,j,k\} \in \cM_2$, thus yielding
	\begin{equation}\label{eq:SystemEP}
	\begin{aligned}
	\left((e_i\otimes e_j + e_j\otimes e_i)^\top \otimes e_k^\top  +  (e_i\otimes e_k + e_k\otimes e_i)^\top \otimes e_j^\top  + (e_j\otimes e_k + e_k\otimes e_j)^\top \otimes e_i^\top  \right)\bH_{\texttt{vect}}&= 0.
	\end{aligned}
	\end{equation}
    Equations as \eqref{eq:SystemEP} can be written in matrix form as:
	\begin{equation}
	\cB\bH_{\texttt{vect}} = 0,
	\end{equation}
	where $\cB$ can be constructed as follows:
	\begin{equation}
	\cB = \begin{bmatrix}
	\left((e_1\otimes e_1 + e_1\otimes e_1)^\top \otimes e_2^\top +  (e_1\otimes e_2 + e_2\otimes e_1)^\top \otimes e_1^\top + (e_1\otimes e_2 + e_2\otimes e_1)^\top \otimes e_1^\top \right) \\ 
	\left((e_2\otimes e_1 + e_2\otimes e_1)^\top \otimes e_2^\top +  (e_2\otimes e_2 + e_2\otimes e_2)^\top \otimes e_1^\top + (e_2\otimes e_1 + e_1\otimes e_2)^\top \otimes e_2^\top \right) \\ 
	\vdots \\
	\left((e_1\otimes e_2 + e_2\otimes e_1)^\top \otimes e_3^\top +  (e_1\otimes e_3 + e_3\otimes e_1)^\top \otimes e_2^\top + (e_2\otimes e_3 + e_3\otimes e_2)^\top \otimes e_1^\top \right) \\ 	
	\end{bmatrix}.
	\end{equation}
	Note that the total number of rows in the matrix $\cB$  is $n(n-1) + \tfrac{n(n-1)(n-2))}{6}$. Next, we find that rows of the matrix $\cB$ are linearly independent and can be written as a linear combination of rows of the matrix $\cA$. Consequently, we can write $\bP\cA = \begin{bmatrix} \cA_1 \\ \cB\end{bmatrix}$, where $\bP$ is an invertible matrix and $\cA_1$ contains a subset of the rows of the matrix $\cA$. Then, after multiplying \eqref{eq:equvi_H_Aform} by the matrix $\bP$ on both sides, we get
	\begin{equation}
	\begin{aligned}
	\bP \cA \bH_{\texttt{vect}} &= 	\bP \cA \tilde\bH_{\texttt{vect}},\\
	\begin{bmatrix}
	\cA_1 \\ \cB
	\end{bmatrix} \bH_{\texttt{vect}} &= \begin{bmatrix}	\cA_1 \\ \cB \end{bmatrix} \tilde\bH_{\texttt{vect}}.
	\end{aligned}
	\end{equation}
	Due to energy-preserving properties of the matrices $\bH$ and $\tilde \bH$, we have $\cB \bH_{\texttt{vect}} =  \cB\tilde \bH_{\texttt{vect}} = 0$. Hence, to ensure $\bH(\bx\otimes \bx) = \tilde{\bH}(\bx\otimes \bx)$ for every $\bx \in \R^n$, provided the matrices $\bH$ and $\tilde \bH$ are energy-preserving, we only require the following condition to be satisfied:
	\begin{equation}\label{eq:constraint_1}
	\begin{aligned}
	\cA_1 \bH_{\texttt{vect}} &= \cA_1  \tilde\bH_{\texttt{vect}},
	\end{aligned}
	\end{equation}
	with the number of rows of the matrix $\cA_1$ to be $$\dfrac{n^2(n+1)}{2} - n - n(n-1) - \dfrac{n(n-1)(n-2))}{6} = \dfrac{n(n+1)(n-1)}{3}.$$
	
	Additionally, we aim to impose an additional structure on the matrix $\tilde \bH$, particularly the one given in \eqref{eq:skew_symmetric_H}. We begin by noticing that due to the particular skew-symmetric structure, we have $\dfrac{n^2(n+1)}{2}$ constraints on the matrix $\tilde \bH$, which can be written as follows:
	\begin{equation}
	\left(e_k^\top\otimes (e_i\otimes e_j + e_j\otimes e_i)\right)\tilde{\bH}_{\texttt{vect}}  = 0,
	\end{equation}
	where $(i,j,k) \in \{1,\ldots,n\}$. We can write these constraints in matrix form as follows:
	\begin{equation}
	\cC \tilde{\bH}_{\texttt{vect}} = 0.
	\end{equation} 
	Combining the above equation with \eqref{eq:constraint_1}, we obtain 
	\begin{equation}\label{eq:solve_H_tH}
	\begin{bmatrix}
	\cA_1 \\ \cC
	\end{bmatrix} \tilde{\bH}_{\texttt{vect}} = \begin{bmatrix} 	
	\cA_1 \tilde{\bH}_{\texttt{vect}} \\ 0 \end{bmatrix}.
	\end{equation}
	Next, note that the rows of the matrices $\cA_1$ and $\cC$ are linearly independent. The reason is that the rows are $\cA_1$ are constructed using 
	\begin{equation}
	\left((e_i^\top \otimes e_j^\top + e_j^\top \otimes e_i^\top) \otimes e_k^\top\right)\qquad {(i,j,k)} \in \cM_2
	\end{equation}
	and the rows are $\cC$ are constructed using 
	\begin{equation}
	e_k^\top\otimes \left((e_i^\top \otimes e_j^\top + e_j^\top \otimes e_i^\top)\right)\qquad {(i,j,k)} \in \cM_1 \cup \cM_2,
	\end{equation}
	and they are linearly independent. Indeed, this follows form the property that $e_k^\top \otimes \ba$ and $\ba \otimes e_k^\top $ are linear independent for every $\ba$ and $k \in \{1,\ldots, n\}$ provided $\ba\neq \alpha e_k$ for $\alpha \in \R$, and the vectors corresponding to $\ba = \alpha e_k$ are not present in the rows of $\cA_1$. Therefore, the matrix $		\begin{bmatrix}
	\cA_1 \\ \cC
	\end{bmatrix}$ is full row-rank with the total number of rows  $$\dfrac{n^2(n+1)}{2} + \dfrac{n(n+1)(n-1)}{3},$$ which is less than (for $n>2$) or equal (for $n\in\{1 , 2\}$) to $n^3$. Hence, a solution always exists to the equation \eqref{eq:solve_H_tH}. This ensures the existence of $\tilde \bH$ in the form \eqref{eq:skew_symmetric_H} such that $\bH(\bx\otimes \bx) = \tilde \bH(\bx\otimes \bx)$ for all $\bx \in \Rn$ provided $\bH$ is energy-preserving. This concludes the proof.
\end{proof}

Finally, we extend \Cref{thm:equvilence_between_H} to the case where $\bH$ is a generalized energy-preserving matrix with respect to an SPD matrix $\bQ$.

\begin{lemma}\label{lemma:equvilence_between_H_gen}
	Consider a generalized energy-preserving matrix $\bH \in \R^{n\times n^2}$ with respect to $\bQ = \bQ^\top > 0$ as given in \eqref{eq:GenEnergyPreservQuadTerm}, i.e., 
	\begin{equation}
	\bx^\top\bQ\bH(\bx\otimes \bx) = 0 \quad \forall \bx \in \Rn.
	\end{equation}
	Then, there exists a matrix $\tilde \bH$ of the following form:
	\begin{equation}
	\tilde \bH = \begin{bmatrix} \bG_1\bQ, \ldots, \bG_n\bQ\end{bmatrix}
	\end{equation}
	with $\bG_i = -\bG_i^\top $ for $i\in \{1,\ldots,n\}$ so that $\bH(\bx\otimes \bx) = \tilde \bH(\bx\otimes \bx)$.
\end{lemma}
\begin{proof}
	We begin by first transforming this problem into the problem considered in \Cref{thm:equvilence_between_H}. Factorizing $\bQ = \bL^\top\bL$, we have
	\begin{align*}
	0 = \bx^\top\bQ\bH(\bx\otimes \bx) &=  \bx^\top\bL^\top\bL\bH(\bx\otimes \bx) ,\\
	&=  \bx^\top\bL^\top\bL\bH(\bL^{-1}\otimes \bL^{-1})(\bL\otimes \bL)(\bx\otimes \bx),\\
	&= \tilde\bx^\top \hat\bH(\tilde\bx\otimes \tilde\bx),
	\end{align*}
	where $\tilde\bx = \bL\bx$ and $\hat\bH = \bL\bH(\bL^{-1}\otimes \bL^{-1})$. From \Cref{thm:equvilence_between_H}, we know that 
	\begin{align*}
	\hat \bH(\tilde \bx \otimes \tilde \bx) = \begin{bmatrix}
	\bJ_1,\ldots, \bJ_n
	\end{bmatrix}  (\tilde \bx \otimes \tilde \bx). 
	\end{align*}
	Back substituting for $\tilde \bx$ and $\hat\bH$, we obtain
	\begin{align*}
	&		\bL\bH(\bL^{-1}\otimes \bL^{-1}) (\bL\bx \otimes \bL\bx) = \begin{bmatrix} \bJ_1,\ldots, \bJ_n \end{bmatrix}  (\bL\bx \otimes \bL\bx) \\
	\Leftrightarrow ~~& \bH(\bL^{-1}\otimes \bL^{-1}) (\bL\bx \otimes \bL\bx) = \bL^{-1}\begin{bmatrix} \bJ_1,\ldots, \bJ_n \end{bmatrix}  (\bL\bx \otimes \bL\bx)\\
	\Leftrightarrow ~~& \bH(\bx \otimes \bx) = \bL^{-1}\begin{bmatrix} \bJ_1,\ldots, \bJ_n \end{bmatrix}  (\bL \otimes \bL)  (\bx \otimes \bx)\\ 
	\Leftrightarrow ~~& \bH(\bx \otimes \bx) = \bL^{-1}\begin{bmatrix} \bJ_1,\ldots, \bJ_n \end{bmatrix}  (\bL \otimes \bI)  (\bI \otimes \bL)   (\bx \otimes \bx).
	\end{align*}
	Now, we focus on the term
	\begin{equation*}
	\bL^{-1}\begin{bmatrix} \bJ_1,\ldots, \bJ_n \end{bmatrix} (\bL \otimes \bI)   (\bI \otimes \bL)  (\bx \otimes \bx).
	\end{equation*}
	Inserting an identity matrix, we have 
	\begin{align*}
	&\bL^{-1}\begin{bmatrix} \bJ_1,\ldots, \bJ_n \end{bmatrix}  (\bL \otimes \bI)   (\bI \otimes \bL^{-\top}) (\bI \otimes \bL^\top ) (\bI \otimes \bL)   (\bx \otimes \bx) \\ 
	& = \bL^{-1}\begin{bmatrix} \bJ_1,\ldots, \bJ_n \end{bmatrix}  (\bL \otimes \bI)   (\bI \otimes \bL^{-\top}) (\bI \otimes \bL^\top\bL )   (\bx \otimes \bx)\\
	& = \bL^{-1}\begin{bmatrix} \bJ_1,\ldots, \bJ_n \end{bmatrix}  (\bL \otimes \bI)   (\bI \otimes \bL^{-\top}) (\bI \otimes \bQ)   (\bx \otimes \bx)\\
	& = \bL^{-1}\begin{bmatrix} \bJ_1,\ldots, \bJ_n \end{bmatrix}  (\bI \otimes \bL^{-\top})  (\bL \otimes \bI)   (\bI \otimes \bQ)   (\bx \otimes \bx) \\
	&\hspace{4cm} (\text{due to the commutative property of} ~ (\bI \otimes \bL^{-\top}) ~\text{and}~ (\bL \otimes \bI))\\
	& = \begin{bmatrix} \hat\bJ_1,\ldots, \hat\bJ_n \end{bmatrix}   (\bL \otimes \bI)   (\bI \otimes \bQ)   (\bx \otimes \bx),
	\end{align*}
	where $\hat\bJ_i = \bL^{-1}\bJ_1\bL^{-\top}$ which are again skew-symmetric matrices. Furthermore, it can also be noticed that 
	\begin{equation*}
	\begin{bmatrix} \hat\bJ_1,\ldots, \hat\bJ_n \end{bmatrix}   (\bL \otimes \bI) = \begin{bmatrix} \bG_1,\ldots, \bG_n \end{bmatrix}, 
	\end{equation*}
	where $\bG_i = \sum\limits_{j=1}^n l_{ji}\hat\bJ_j$, which are also skew-symmetric. Combining all these components, we have 
	\begin{equation*}
	\bH(\bx \otimes \bx) = \begin{bmatrix} \bG_1,\ldots, \bG_n \end{bmatrix} (\bI \otimes \bQ) = \begin{bmatrix} \bG_1\bQ,\ldots, \bG_n\bQ \end{bmatrix}.
	\end{equation*}
	This concludes the proof. 
\end{proof}
\end{document}